\documentclass{article}

\PassOptionsToPackage{numbers, compress}{natbib}

\usepackage[preprint]{neurips_2023}

\usepackage{subcaption}

\usepackage{amssymb}

\title{Beyond Black-Box Advice: Learning-Augmented Algorithms for MDPs with Q-Value Predictions}

\author{
  Tongxin~Li \\
  School of Data Science \\
  CUHK-SZ, China \\
  \texttt{litongxin@cuhk.edu.cn}
  \And 
  Yiheng Lin \\
  Computing + Mathematical Sciences \\
  Caltech, USA \\
  \texttt{yihengl@caltech.edu}
  \And
  Shaolei Ren \\
  Electrical \& Computer Engineering \\
  UC Riverside, USA \\
  \texttt{shaolei@ucr.edu}
  \And
  Adam Wierman \\
  Computing + Mathematical Sciences \\
  Caltech, USA \\
  \texttt{adamw@caltech.edu}
}

\usepackage{graphicx}
\graphicspath{{Figs/}}
\usepackage{comment}
\usepackage{caption}
\usepackage[utf8]{inputenc} 
\usepackage[T1]{fontenc}  
\usepackage{hyperref}   
\usepackage{url}   
\usepackage{booktabs}   
\usepackage{amsfonts}      
\usepackage{nicefrac}     
\usepackage{microtype}     
\usepackage[linesnumbered,ruled,vlined,onelanguage,algo2e]{algorithm2e}
\usepackage{makecell}
\usepackage{xcolor}
\usepackage[framemethod=tikz]{mdframed}
\usepackage{mathtools}
\usepackage{amsmath,amsfonts,amssymb,amscd,xspace}
\usepackage{amsthm}
\usepackage[shortlabels]{enumitem}
\usepackage{bbm}
\usepackage{bm}
\usepackage{booktabs}
\usepackage{algorithm}
\usepackage{algpseudocode}

\makeatletter
\newenvironment{myprocedure}[1][htb]{%
    \renewcommand{\ALG@name}{Procedure}
   \begin{algorithm}[#1]%
  }{\end{algorithm}}
\makeatother

\graphicspath{{Figs/}}

\usepackage{graphicx}

\newtheorem{assumption}{Assumption}

\newcommand{\nt}{T}
\newcommand{\nh}{H}

\newcommand{\decayfactor}{\lambda}

\DeclareMathOperator*{\argmin}{arg\,min}

\DeclareMathOperator*{\arginf}{arg\,inf}
\newcommand{\norm}[1]{\left\lVert#1\right\rVert}
\newcommand{\MPC}{\mathsf{MPC}}

\newcommand{\DIAG}{\mathsf{diag}}

\newcommand{\TV}{\mathsf{TV}}
\newcommand{\abs}[1]{\left\lvert#1\right\rvert}

\newcommand{\yiheng}[1]{{\textcolor{purple}{(Yiheng says:  #1)}}}

\newcommand{\ouralg}{{\textsc{PROP}}\xspace}

\newtheorem{theorem}{Theorem}
\numberwithin{theorem}{section}
\newtheorem{lemma}{Lemma}

\numberwithin{corollary}{section}

\newtheorem{definition}{Definition}

\begin{document}

\maketitle

\begin{abstract}
We study the tradeoff between consistency and robustness in the context of a single-trajectory time-varying Markov Decision Process (MDP) with untrusted machine-learned advice. Our work departs from the typical approach of
treating advice as coming from black-box sources by instead considering a setting where additional information about  how the advice is generated is available.
We prove a first-of-its-kind consistency and robustness tradeoff given Q-value advice under a general MDP model that includes both continuous and discrete state/action spaces.
Our results highlight that utilizing Q-value advice enables dynamic pursuit
of the better of machine-learned advice and a robust baseline, thus result in near-optimal performance guarantees, which provably 
 improves what can be obtained solely with black-box advice.
\end{abstract}

\paragraph{Revision note.}
This version corrects published Lemma 4 and the proof of Theorem 5.4 from arXiv version 2. The revised Grey Box Procedure accounts for transition noise in the sampled temporal difference statistic and retains the stated asymptotic consistency and robustness conclusion under the explicit conditions below. The correction was prompted by Anders Wikum, whom we thank for identifying the expectation versus pathwise issue.


\section{Introduction}
\label{sec:intro}

Machine-learned predictions and hand-crafted algorithmic advice are both crucial in online decision-making problems, driving a growing interest in \textit{learning-augmented algorithms}~\cite{mitzenmacher2022algorithms,li2023learning} that exploit the benefits of predictions to improve the performance for typical problem instances while bounding
the worst-case performance
~\cite{purohit2018improving,christianson2022chasing}. To this point, the study of learning-augmented algorithms has primarily viewed machine-learned advice as potentially untrusted information generated by black-box     
models. Yet, in many real-world problems, additional knowledge of the machine learning models used to produce advice/predictions is often available and can potentially improve the performance of learning-augmented algorithms.

 

A notable example that motivates our work is the problem of minimizing costs (or maximizing rewards) in a single-trajectory Markov Decision Process (MDP). More concretely, 
a value-based machine-learned policy $\widetilde{\pi}$ can be queried to provide suggested actions as advice to the agent at each step ~\cite{nachum2017bridging,yang2020function,golowich2022can}.
Typically, the suggested actions are chosen to minimize (or  maximize, in case of rewards) estimated cost-to-go functions (known as {Q}-value predictions) based on the current state.

Naturally,  in addition to suggested actions, the {Q}-value function itself can also provide additional information 
(e.g., the long-term impact of choosing a certain action) potentially useful to the design of a learning-augmented algorithm.
Thus, this leads to two different designs for learning-augmented algorithms in MDPs: \textit{black-box} algorithms and \textit{grey-box} algorithms. 
%
%
A learning-augmented algorithm using $\widetilde{\pi}$ is black-box if $\widetilde{\pi}$ provides only the suggested action $\widetilde{u}$ to the learning-augmented algorithm, whereas it is value-based (a.k.a., grey-box) if $\widetilde{\pi}$ provides an estimate of the {Q}-value function $\smash{\widetilde{Q}}$ (that also implicitly includes
a suggested action $\widetilde{u}$ obtained by minimizing $\smash{\widetilde{Q}}$) to the learning-augmented algorithm.


Value-based policies $\widetilde{\pi}$ often perform well empirically in stationary environments in practice \cite{nachum2017bridging,yang2020function}.  However, they may not have performance guarantees in all environments
and can perform poorly at times due to a variety of factors, such as non-stationary environments~\cite{wei2021non,mao2021near,luo2022dynamic,zhao2022dynamic}, policy collapse~\cite{scheller2020sample}, sample inefficiency~\cite{botvinick2019reinforcement}, and/or when  training data is biased~\cite{bai2019model}. As a consequence, such policies often are referred to as ``untrusted advice'' in the literature on learning-augmented algorithms, where the notion of ``untrusted'' highlights the lack of performance guarantees. In contrast, recent studies in competitive online control~\cite{shi2020online,goel2022competitive,li2022robustness,sabag2022optimal,lin2022bounded,li2021information,li2021learning} have begun to focus on worst-case analysis and provide control policies $\overline{\pi}$ with strong performance guarantees even in adversarial settings, referred to as \emph{robustness}, i.e., $\overline{\pi}$ provides ``trusted advice.'' Typically, the goal of a learning-augmented online algorithm~\cite{mitzenmacher2022algorithms,purohit2018improving} is to perform nearly as well as the untrusted advice when the machine learned policy performs well, a.k.a., achieve \emph{consistency}, while also ensuring worst-case robustness. Combining the advice of an untrusted machine-learned policy $\widetilde{\pi}$ and a robust policy $\overline{\pi}$
 naturally leads to a tradeoff between consistency and robustness. In this paper, we explore this tradeoff in a time-varying MDP setting and seek to answer the following key question for learning-augmented online algorithms: 


\textit{Can Q-value advice from an untrusted machine-learned policy, $\widetilde{\pi}$, in a \textbf{grey-box} scenario provide more benefits than the \textbf{black-box} action advice generated by $\widetilde{\pi}$ in the context of \textbf{consistency and robustness tradeoffs} for MDPs?}
 

\subsection{Contributions}

We answer the question above in the affirmative by presenting and analyzing a unified projection-based learning-augmented online algorithm (\textsc{PRO}jection \textsc{P}ursuit policy, simplified as \ouralg in Algorithm~\ref{alg:ppp}) that combines action feedback from a trusted, robust policy $\overline{\pi}$ with an untrusted ML policy $\widetilde{\pi}$.
In addition to offering a consistency and robustness tradeoff for MDPs with black-box advice, our work moves beyond the black-box setting. Importantly, by considering the grey-box setting, the design of \ouralg demonstrates that the \textit{structural information} of the untrusted machine-learned advice can be leveraged to determine the trust parameters dynamically, 
which would otherwise be challenging (if not impossible) in a black-box setting.
To our best knowledge, \ouralg is the first-of-its-kind learning-augmented algorithm that applies to general MDP models, which allow continuous or discrete state and action spaces.

Our main results characterize the tradeoff between consistency and robustness for both black-box and grey-box settings in terms of the ratio of expectations, $\mathsf{RoE}$, built upon the traditional consistency and robustness metrics in~\cite{purohit2018improving,wei2020optimal,banerjee2020improving,christianson2022chasing} for the competitive ratio.
We show in Theorem~\ref{thm:black_consistency_robustness} that for the black-box setting, $\ouralg$  
is $\left(1+\mathcal{O}((1-\lambda)\gamma)\right)$-consistent and $(\mathsf{ROB}+\mathcal{O}(\lambda\gamma))$-robust where $0\leq\lambda\leq 1$ is a hyper-parameter. Moreover, for the black-box setting, $\ouralg$
cannot be both $\left(1+o(\lambda\gamma)\right)$-consistent and $(\mathsf{ROB}+o((1-\lambda)\gamma))$-robust for any $0\leq\lambda\leq 1$ where $\gamma$ is the diameter of the action space. In sharp contrast, by using the confidence-corrected robustness
budget in \ouralg with Q-value advice (grey-box setting), \ouralg is $(1+o(1))$-consistent and $(\mathsf{ROB}+o(1))$-robust under the conditions of Theorem~\ref{thm:grey_robustness_consistency}.

Our result highlights the benefits of exploiting the additional
information informed by the estimated {Q}-value functions,
showing that the ratio of expectations
can approach the better of the two policies $\widetilde{\pi}$ and
$\overline{\pi}$ for any single-trajectory time-varying, and even possibly adversarial environments --- if the value-based policy $\widetilde{\pi}$ is near-optimal, then the worst-case $\mathsf{RoE}(\ouralg)$ can approach $1$ as governed by a consistency parameter; otherwise, $\mathsf{RoE}(\ouralg)$ can be bounded by the ratio of expectations of $\overline{\pi}$ subject to an additive term $o(1)$ that decreases when the time horizon $\nt$ increases.

A key technical contribution of our work is
to provide the first quantitative
characterization of the consistency and robustness tradeoff
for a learning-augmented algorithm (\ouralg) in a general MDP model, under both
standard black-box and novel grey-box settings.  
Importantly, \ouralg is able to leverage
a  broad class of robust policies, called 
\textit{Wasserstein robust} policies, which
generalize the well-known 
contraction principles
that are satisfied by various robust policies~\cite{tu2022sample}
and have been used to derive regrets for online control~\cite{lin2022bounded,tsukamoto2021contraction}.
A few concrete examples of Wasserstein robust policies applicable for \ouralg are provided in Table~\ref{table:baseline}(Section~\ref{sec:wasserstein-robust}).


\subsection{Related Work}
\label{sec:related_work}

\textbf{Learning-Augmented Algorithms with Black-Box Advice.}
The concept of integrating black-box machine-learned guidance into online algorithms was initially introduced by~\cite{mahdian2012online}. \cite{purohit2018improving} coined terms ``robustness" and ``consistency" with formal mathematical definitions based on the competitive ratio.
Over the past few years, the consistency and robustness approach has gained widespread popularity and has been utilized to design online algorithms with black-box advice for various applications, including ski rental~\cite{purohit2018improving,wei2020optimal,banerjee2020improving}, caching~\cite{rohatgi2020near,lykouris2021competitive,im2022parsimonious}, bipartite matching~\cite{antoniadis2020online}, online covering~\cite{bamas2020primal,anand2022online}, convex body chasing~\cite{christianson2022chasing}, nonlinear quadratic control~\cite{li2023certifying}. 
The prior studies on learning-enhanced algorithms have mainly focused on creating meta-strategies that combine online algorithms with black-box predictions, and typically require manual setting
of a trust hyper-parameter  
to balance consistency and robustness.
A more recent learning-augmented algorithm in~\cite{li2023certifying} investigated the balance between competitiveness and stability in nonlinear control in a black-box setting. However, this work limits the robust policy to a linear quadratic regulator and does not provide a theoretical basis for the selection of the trust parameters.~\cite{diakonikolas2021learning} generalized the black-box advice setting by considering distributional advice.



\textbf{Online Control and Optimization with Structural Information.} Despite the lack of a systematic analysis, recent studies have explored the usage of structural information in online control and optimization problems. Closest to our work,~\cite{golowich2022can} considered a related setting where the {Q}-value function is available as advice, and shows that such information can be utilized to reduce regret in a tabular MDP model.
In contrast, our analysis applies to more general models that allow continuous state/action spaces. In~\cite{li2022robustness}, the dynamical model and the 
predictions of disturbances in a linear control system are shown to be useful in achieving a near-optimal consistency and robustness tradeoff.
The predictive optimization problem solved by MPC~\cite{lin2021perturbation,linbounded,goel2022competitive,hoeller2020deep} can be regarded as a special realization of grey-box advice, where an approximated cost-to-go function is constructed from structural information that includes the (predicted) dynamical model, costs, and disturbances.

\vspace{3pt}

\noindent\textbf{MDP with External Feedback.} Feedback from external sources such as control baselines~\cite{cheng2019control,brunke2022safe}, visual explanations~\cite{guan2021widening}, and human experts~\cite{christiano2017deep,macglashan2017interactive,gao2022scaling} is often available in MDP.
This external feedback can be beneficial for various purposes, such as ensuring safety~\cite{berkenkamp2017safe}, reducing variance~\cite{cheng2019control}, training human-like chatbots~\cite{christiano2017deep}, and enhancing overall trustworthiness~\cite{xu2022trustworthy}, among others. The use of control priors has been proposed by~\cite{cheng2019control} as a way to guarantee the Lyapunov stability of the training process in reinforcement learning. They used the Temporal-Difference method to tune a coefficient that combines a RL policy and a control prior, but without providing a theoretical foundation. Another related area is transfer learning in RL, where external {Q}-value advice from previous tasks can be adapted and utilized in new tasks. Previous research has shown that this approach can outperform an agnostic initialization of Q, but these results are solely based on empirical observations and lack theoretical support~\cite{taylor2009transfer,higgins2017darla,chen2022transferred}.


\vspace{3pt}


\section{Problem Setting}
\label{sec:model}


We consider a finite-horizon, single-trajectory, time-varying MDP with $\nt$ discrete time steps. The state space $\mathcal{X}$ is a subset of a normed vector space embedded with a norm $\|\cdot\|_{\mathcal{X}}$. The actions are chosen from a convex and compact set $\mathcal{U}$ in a normed vector space characterized by some norm $\|\cdot\|_{\mathcal{U}}$. Notably, 
$\mathcal{U}$ can represent either continuous actions or the probability distributions used when choosing actions from a finite set.\footnote{The action space $\mathcal{U}$ is assumed to be a continuous, convex, and compact set for more generality. When the actions are discrete, $\mathcal{U}$ can be defined as the set of all probability distributions on a finite action space. We relegate the detailed discussions in Appendix~\ref{appendix:Wasserstein-TV-distance} and~\ref{appendix:MDP-proof}.}
The diameter of the action space $\mathcal{U}$ is denoted by $\gamma\coloneqq \sup_{u,v\in\mathcal{U}}\|u-v\|_{\mathcal{U}}$.
Denote $[\nt]\coloneqq \{0,\ldots,\nt-1\}$.
For each time step $t\in [\nt]$, let ${P}_t:\mathcal{X}\times\mathcal{U}\rightarrow \mathcal{P}_{\mathcal{X}}$ be the transition probability, where $\mathcal{P}_{\mathcal{X}}$ is a set of probability measures on $\mathcal{X}$. We consider time-varying costs $c_t:\mathcal{X}\times\mathcal{U}\rightarrow\mathbb{R}_+$, while rewards can be treated similarly by adding a negative sign. An initial state $x_0\in\mathcal{X}$ is fixed.
This MDP model is compactly represented by $\mathsf{MDP}(\mathcal{X},\mathcal{U},\nt,{P},c)$. 


The goal of a policy in this MDP setting is to minimize the total cost over all $\nt$ steps. The policy agent has no access to the full MDP. At each time step $t\in [\nt]$, only the incurred cost value $c_t(x_t,u_t)$ and the next state $x_{t+1}\sim P_t(\cdot|x_t,u_t)$ are revealed to the agent after playing an action $u_t\in \mathcal{U}$. We denote a policy by $\pi=\left(\pi_t:t\in [\nt]\right)$ where each $\pi_t:\mathcal{X}\rightarrow\mathcal{U}$ chooses an action $u_t$ when observing $x_t$ at step $t\in [\nt]$. Note that our results can be generalized to the setting when $\pi_t$ is stochastic and outputs a probability distribution on $\mathcal{U}$.
Given $\mathsf{MDP}(\mathcal{X},\mathcal{U},\nt,{P},c)$, we consider an optimization with time-varying costs and transition dynamics. Thus, our goal is to find a policy $\pi$ that minimizes the following expected total cost:
\begin{equation}
\label{eq:offline}
J(\pi)\coloneqq \mathbb{E}_{{P},\pi}\Big[\sum_{t\in [\nt]}c_t\left(x_t,\pi_t(x_t)\right) \Big]
\end{equation}
where the randomness in $\mathbb{E}_{{P},\pi}$ is from the transition dynamics ${P}=\left(P_t:t\in [\nt]\right)$ and the policy $\pi=\left(\pi_t:t\in [\nt]\right)$. We focus our analysis on the expected dynamic regret and the ratio of expectations, defined below, as the performance metrics for our policy design.

\begin{definition}[Expected dynamic regret]
 \label{def:dr}
Given $\mathsf{MDP}(\mathcal{X},\mathcal{U},\nt,{P},c)$, the (expected) dynamic regret
of a policy $\pi=\left(\pi_t:t\in [\nt]\right)$ is defined as the difference between the expected cost induced by the policy $\pi$, $J(\pi)$ in~\eqref{eq:offline}, and the optimal expected cost $J^{\star}\coloneqq \inf_{\pi} J(\pi)$, i.e., $ \mathsf{DR}(\pi)\coloneqq J(\pi)-J^{\star}$.
 \end{definition}

Dynamic regret is a more general (and often more challenging to analyze) measure than classical static regret, which has been mostly used for stationary environments~\cite{auer2008near,azar2017minimax}.  The following definition of the ratio of expectations~\cite{borodin2005online,devanur2009adwords} will be used as an alternative performance metric in our main results.
\begin{definition}[Ratio of expectations]
 \label{def:roe}
Given $\mathsf{MDP}(\mathcal{X},\mathcal{U},\nt,{P},c)$, the ratio of expectations 
of a policy $\pi=\left(\pi_t:t\in [\nt]\right)$ is defined as $\mathsf{RoE}(\pi)\coloneqq {J(\pi)}/{J^{\star}}$ where $J(\pi)$ and $J^{\star}$ are the same as in Definition~\ref{def:dr}.
 \end{definition}

Dynamic regret and the ratio of expectations defined above also depend on the error of the untrusted ML advice; we make this more explicit in Section~\ref{sec:q-value_advice}. Next, we state the following
continuity assumption, which is standard in MDPs with continuous action and state spaces~\cite{agarwal2019online,hazan2019provably, gradu2020non}. Note that our analysis can be readily adapted to general H\"older continuous costs with minimal modifications.
\begin{assumption}[Lipschitz costs]
\label{ass:Lip}
For any time step $t\in [\nt]$, the cost function $c_t:\mathcal{X}\times\mathcal{U}\rightarrow\mathbb{R}_+$ is Lipschitz continuous with a Lipschitz constant $L_{\mathrm{C}}<\infty$, i.e., for any $t\in [\nt]$, $\left|c_t(x,u)-c_t(x',u')\right| \leq L_C\left(\|x-x'\|_{\mathcal{X}} + \|u-u'\|_{\mathcal{U}}\right)$. Moreover, $0<c_t(x,u)<\infty$ for all $t\in [\nt]$, $x\in\mathcal{X}$, and $u\in\mathcal{U}$.
\end{assumption}

\section{Consistency and Robustness in MDPs}
\label{sec:robustness_consistency}


\begin{figure}[h]
    \centering
\includegraphics[scale=0.166]{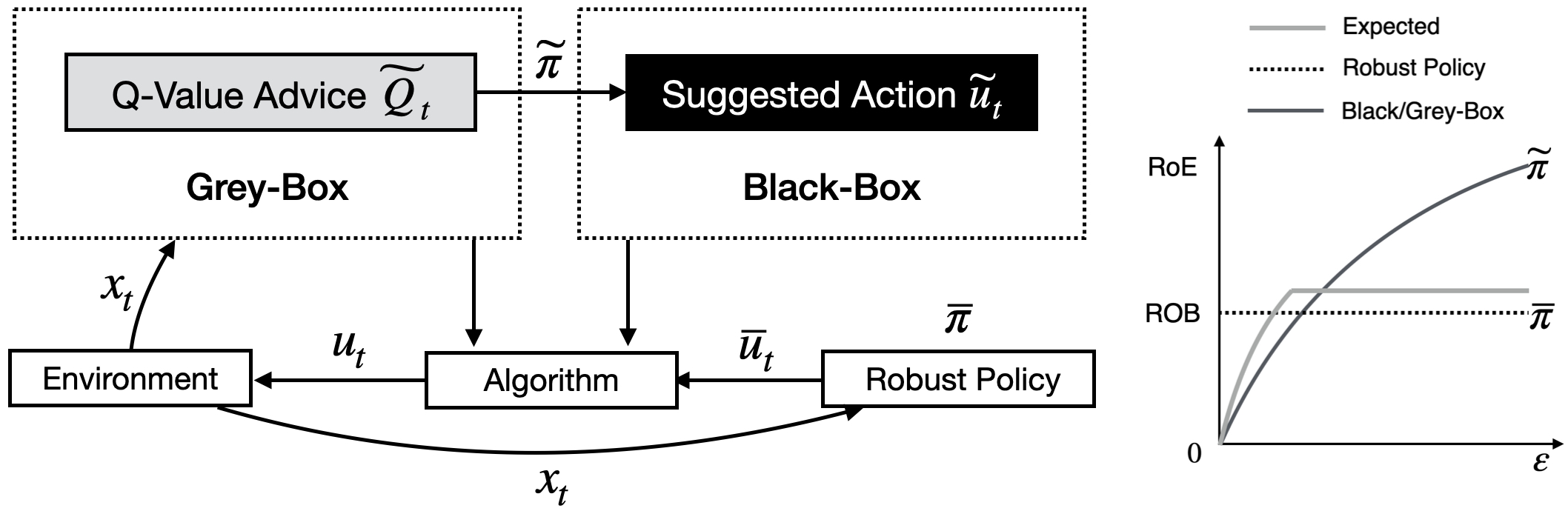}
\caption{\textit{Left}: Overview of settings in our problem. \textit{Right}: consistency and robustness tradeoff, with $\mathsf{RoE}$ and $\varepsilon$ defined in Definition~\ref{def:roe} and Equation~\eqref{eq:def_epsilon}.}
\label{fig:system}
\end{figure}

Our objective is to achieve a balance between the worst-case guarantees on cost minimization in terms of dynamic regret provided by a robust policy, $\overline{\pi}$, and the average-case performance of a valued-based policy, $\widetilde{\pi}$, in the context of $\mathsf{MDP}(\mathcal{X},\mathcal{U},\nt,{P},c)$. In particular, we denote by $\mathsf{ROB}\geq 1$ a ratio of expectation bound of the robust policy $\overline{\pi}$ such that the worst case $\mathsf{RoE}(\overline{\pi})\leq \mathsf{ROB}$. In the learning-augmented algorithms literature, these two goals are referred to as consistency and robustness~\cite{purohit2018improving,mitzenmacher2022algorithms}. Informally, robustness refers to the goal of ensuring worst-case guarantees on cost minimization comparable to those provided by $\overline{\pi}$ and consistency refers to ensuring performance nearly as good as $\widetilde{\pi}$ when $\widetilde{\pi}$ performs well (e.g., when the instance is not adversarial).  Learning-augmented algorithms seek to achieve consistency and robustness by combining $\overline{\pi}$ and $\widetilde{\pi}$, as illustrated in Figure~\ref{fig:system}.  

Our focus in this work is to design robust and consistent algorithms for two types of advice: black-box advice and grey-box advice.  The type of advice that is nearly always the focus in the learning-augmented algorithm literature is black-box advice --- only
providing a suggested action $\widetilde{u}_t$ without additional information.
In contrast, on top of the action $\smash{\widetilde{u}_t}$,
grey-box advice can also reveal the internal state of the learning algorithm, e.g., the {Q}-value $\smash{\widetilde{Q}_t}$ in our setting. This contrast is illustrated in Figure~\ref{fig:system}. 

Compared to black-box advice, grey-box advice has received much less attention in the literature,
despite its potential to improve tradeoffs between consistency and robustness as recently shown in~\cite{diakonikolas2021learning,li2022robustness}.  
Nonetheless, the extra information on top of the suggested action in a grey-box setting
potentially allows the learning-augmented algorithm to 
make a better-informed decision based on the advice, 
thus achieving a better tradeoff between consistency and robustness than  otherwise possible. 


In the remainder of this section, we discuss the robustness properties for the algorithms we consider in our learning-augmented framework (Section~\ref{sec:wasserstein-robust}), 
and introduce the notions of consistency 
in our grey-box  and black-box  models in Section~\ref{sec:q-value_advice}.  

\subsection{Locally Wasserstein-Robust Policies}
\label{sec:wasserstein-robust}



We begin with constructing a novel notion of robustness for our learning-augmented framework based on the Wasserstein distance as follows.  Denote the robust policy by $\overline{\pi} \coloneqq (\overline\pi_t: t\in [\nt])$, where each $\overline\pi_t$ maps a system state to a deterministic action (or a probability of actions in the stochastic setting).
Denote by $\rho_{t_1:t_2}(\rho)$ the joint distribution of the state-action pair $(x_t,u_t)\in\mathcal{X}\times \mathcal{U}$ at time $t_2\in [\nt]$ when implementing the baselines $\overline{\pi}_{t_1},\ldots,\overline{\pi}_{t_2}$ consecutively with an initial state-action distribution $\rho$. We use $\|\cdot\|_{\mathcal{X}\times\mathcal{U}}\coloneqq\|\cdot\|_{\mathcal{X}}+\|\cdot\|_{\mathcal{U}}$ as the included norm for the product space $\mathcal{X}\times \mathcal{U}$. Let ${W}_p(\mu,\nu)$ denote the Wasserstein $p$-distance between distributions $\mu$ and $\nu$ whose support set is $\mathcal{X}\times \mathcal{U}$:
\begin{equation*}
 {W}_p(\mu,\nu) \coloneqq \left(\inf_{J\in \mathcal{J}(\mu,\nu)} \int \|(x,u)-(x',u')\|_{\mathcal{X}\times\mathcal{U}}^p \mathrm{d} J\left((x,u),(x',u')\right)\right)^{1/p}  
\end{equation*}
where $p\in [1,\infty)$ and $\mathcal{J}(\mu,\nu)$ denotes a set of all joint distributions $J$ with a support set $ \mathcal{X}\times \mathcal{U}$ that have marginals $\mu$ and $\nu$. 
Next, we define a robustness condition for our learning-augmented framework. 

\begin{table*}[t]
\caption{Examples of models covered in this paper and the associated control baselines. For the right column, bounds on the ratio of expectations $\mathsf{RoE}$ are exemplified, where $\mathsf{ROB}$ is defined in Section~\ref{sec:robustness_consistency} and $\mathcal{O}$ omits inessential constants.
}
    \footnotesize
\renewcommand{\arraystretch}{1.3}
    \centering
    \begin{tabular}{c|c|c}
\specialrule{.09em}{.1em}{.1em} 
    Model  & \textbf{Robust Baseline  $\overline{\pi}$} & $\mathsf{RoE}$\\
        \hline    \hline
Time-varying MDP (Our General Model) & Wasserstein Robust Policy (Definition~\ref{def:robust}) & $\mathsf{ROB}$ \\
    \hline
    Discrete MDP (Appendix~\ref{appendix:Wasserstein-TV-distance})  & Any  Policy that Induced a Regular Markov Chain & --- \\
   Time-Varying LQR (Appendix~\ref{sec:mpc_baseline}) & MPC with Robust Predictions (Algorithm~\ref{alg:mpc-baseline}) & $\mathcal{O}(1)$ \\
\specialrule{.09em}{.1em}{.1em} 
    \end{tabular}
\label{table:baseline}
\end{table*}

\begin{definition}[$r$-locally $p$-Wasserstein robustness] \label{def:robust}  A policy $\overline{\pi}=(\pi_t:t\in [\nt])$ is \textbf{$r$-locally $p$-Wasserstein-robust} if for any $0\leq t_1\leq t_2<\nt$ and any pair of 
 state-action distributions $\rho,\rho'$ where the  the $p$-Wasserstein distance between them is bounded by $W_p(\rho,\rho')\leq r$,
for some radius $r>0$, the following inequality holds: 
\begin{align} 
\label{eq:def_robustness}
{W}_p\left(\rho_{t_1:t_2}(\rho),\rho_{t_1:t_2}(\rho')\right)\leq & s(t_2-t_1) {W}_p\left(\rho,\rho'\right)\end{align} for some function $s:[\nt]\rightarrow \mathbb{R}_+$ satisfying $\sum_{t\in [\nt]}s(t)\leq C_{s}$ where $C_{s}>0$ is a constant.  \end{definition}

\begin{figure}[h]
    \centering
\includegraphics[scale=0.32]{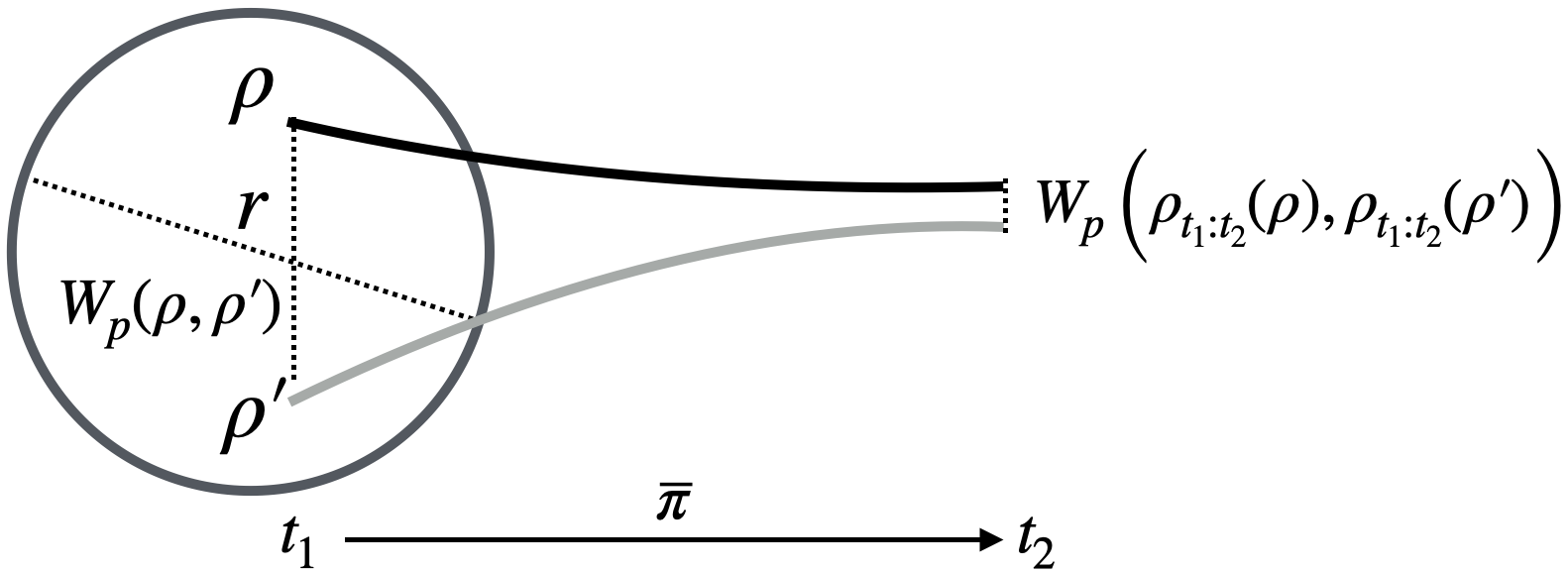}
\caption{An illustration of an \textit{$r$-locally $p$-Wasserstein-robust} policy.}
\label{fig:robust}
\end{figure}

Our robustness definition is naturally more relaxed than the usual contraction property in the control/optimization literature~\cite{tsukamoto2021contraction,lin2021perturbation} ---  if any two different state-action distributions converge exponentially with respect to the Wasserstein $p$-distance, then a policy $\overline{\pi}$ is \textit{$r$-locally $p$-Wasserstein-robust}. This is illustrated in Figure~\ref{fig:robust}. Note that,
although the Wasserstein robustness in Definition~\ref{def:robust} well captures a variety
of distributional robustness metrics such as the total variation robustness defined on finite state/action spaces,
 it can also be further generalized to other metrics for probability distributions.


As shown in Appendix~\ref{sec:application} (provided in the supplementary material), by establishing a connection between the Wasserstein distance and the total variation metric, any policy that induces a regular Markov chain satisfies the fast mixing property and the state-action distribution will converge with respect to the total variation distance to a stationary distribution~\cite{ross1995stochastic}. A more detailed discussion can be found in Appendix~\ref{appendix:Wasserstein-TV-distance}. Moreover, the Wasserstein-robustmess in Definition~\ref{def:robust} 
includes
a set of contraction properties in control theory as special cases. For example, for a locally Wasserstein-robust policy, if the transition kernel ${P}$ and the baseline policy $\overline{\pi}$ are deterministic, then the state-action distributions become point masses,
reducing Definition~\ref{def:robust} to 
a state-action perturbation bound in terms of the $\ell_2$-norm when implementing the policy $\overline{\pi}$ from different starting states
\cite{lin2021perturbation,lin2022bounded}.


The connections discussed above highlight the existence
of several well-known robust policies that satisfy Definition~\ref{def:robust}.  Besides the case of discrete MDPs discussed in Appendix~\ref{appendix:Wasserstein-TV-distance}, another prominent example is model predictive control (MPC), for which robustness follows from the results in~\cite{lin2022bounded} (see Appendix~\ref{sec:mpc_baseline} for details). The model assumption below will be useful in our main results.
\begin{assumption}
\label{ass:robustness}
There exists a $\gamma$-locally $p$-Wasserstein-robust baseline control policy (Definition~\ref{def:robust}) $\overline{\pi}$ for some $p\geq 1$, where $\gamma$ is the diameter of the action space $\mathcal{U}$.
\end{assumption}








\subsection{Consistency and Robustness for RoE}
\label{sec:q-value_advice}



In parallel with the notation of ``consistency and robustness'' in the existing literature on learning-augmented algorithms~\cite{purohit2018improving,mitzenmacher2022algorithms},
we define a new metric of consistency and robustness in terms of $\mathsf{RoE}$. To do so, we first introduce an optimal policy $\pi^{\star}$.
Based on $\mathsf{MDP}(\mathcal{X},\mathcal{U},\nt,{P},c)$, let $\pi_{t}^{\star}=(\pi_{t}^{\star}:t\in [\nt])$ denote the optimal policy at each time step $t\in [\nt]$, whose optimal {Q}-value function is
\begin{equation*}
   Q_{t}^{\star}(x,u)\coloneqq\inf_{\pi}\mathbb{E}_{{P},\pi}\left[\sum_{\tau=t}^{\nt-1}c_{\tau}\left(x_{\tau},u_{\tau}\right)\Big |x_t=x,u_t=u\right],
\end{equation*}
where $\mathbb{E}_{{P},\pi}$ denotes an expectation with respect to the randomness of the trajectory $\{(x_t, u_t):t\in [\nt]\}$ obtained by following a policy $\pi$ and the transition probability ${P}$ at each step $t\in [\nt]$.
The Bellman optimality equations can then be expressed as
\begin{align}
\label{eq:bellman_optimal}
  Q_{t}^{\star}(x,u)=&\left(c_t + \mathbb{P}_t V^{\star}_{t+1}\right)\left(x,u\right), & 
 & V^{\star}_{t}(x) = \inf_{v\in\mathcal{U}}Q_{t}^{\star}(x,v), & & V_{\nt}^{\star}(x)=0
\end{align}
for all $(x,u)\in\mathcal{X}\times\mathcal{U}$, $t\in [\nt]$ and $t\in [\nt]$, where we write $\left(\mathbb{P}_t V^{\star}\right)(x,u)\coloneqq \mathbb{E}_{x'\sim{P}_{t}(\cdot|x,u)}\left[V^{\star}(x')\right]$. 
This indicates that for each time step $t\in [\nt]$, $\pi_t^{\star}$ is the greedy policy with respect to its optimal {Q}-value functions $(Q_{t}^{\star}:t\in [\nt])$.
Note that for any $t\in [\nt]$, $Q_{t}^{\star}(x,u)=0$. 
Given this setup, the value-based policies $\widetilde{\pi}\coloneqq \left(\widetilde{\pi}_t: t\in [\nt]\right)$ take the following form. For any $t\in [\nt]$, a value-based policy $\widetilde{\pi}_t:\mathcal{X}\rightarrow \mathcal{U}$ produces an action $\smash{\widetilde{u}_t \in \argmin_{v\in\mathcal{U}} \widetilde{Q}_{t}\left(x_t,v\right)}$ by minimizing an estimate of the optimal {Q}-value function $\smash{\widetilde{Q}_{t}}$.

We make the following assumption on the machine-learned untrusted policy $\widetilde{\pi}$ and the {Q}-value advice. 
\begin{assumption}
\label{ass:Q-Lip}
The machine-learned untrusted policy $\widetilde{\pi}$ is value-based.
The {Q}-value advice $\widetilde{Q}_{t}:\mathcal{X}\times \mathcal{U}\rightarrow\mathbb{R}$ is Lipschitz continuous with respect to $u\in\mathcal{U}$ for any $x\in\mathcal{X}$, with a Lipschitz constant $L_Q>0$ for all $t\in [\nt]$. Moreover, there is a horizon-dependent bound $q_{\nt}=o(\nt)$ such that $\sup_{t\in[\nt]}\|\widetilde Q_t-Q_t^\star\|_\infty\leq q_{\nt}$ uniformly over the admissible advice class.
\end{assumption}

For the Grey-Box result, robustness is taken over the admissible advice class in Assumption~\ref{ass:Q-Lip}.

We can now define a consistency measure for {Q}-value advice $\widetilde{Q}_{t}$, which measures the error of the estimates of the {Q}-value functions due to approximation error and time-varying environments, etc.
Let $p\in (0,\infty]$. Fix a sequence of distributions $\rho=(\rho_t:t\in [\nt])$ whose support set is $\mathcal{X}\times\mathcal{U}$ and let $\phi_t$ be the marginal distribution of $\rho_t$ on $\mathcal{X}$.  
We define a quantity representing the error of the {Q}-value advice
\begin{equation}
\label{eq:def_epsilon}
 \varepsilon (p,\rho)\coloneqq \sum_{t\in [\nt]}\left(\Big\|\widetilde{Q}_{t}-Q_{t}^{\star}\Big\|_{p,\rho_t} + \Big\|\inf_{v\in\mathcal{U}}\widetilde{Q}_{t}-\inf_{v\in\mathcal{U}}Q_{t}^{\star}\Big\|_{p,\phi_t}\right)
\end{equation}
where $\smash{\|\cdot\|_{p,\rho}\coloneqq\left(\int\left|\cdot\right|^p \mathrm{d}{\rho}\right)^{1/p}}$ denotes the $L_{p,\rho}$-norm. A policy with {Q}-value functions $\{Q_t:t\in [\nt]\}$ is said to be \textit{$(\varepsilon,p,\rho)$-consistent} if there exists an $\varepsilon$ satisfying~\eqref{eq:def_epsilon}. 
In addition, a policy is  $(0,\infty)$-consistent if $\widetilde{Q}_{t}$ is a Lebesgue-measurable function for all $t\in [\nt]$ and $(\infty,\varepsilon)$-consistent if the $L_{\infty}$-norm satisfies $\sum_{t\in [\nt]}\|\widetilde{Q}_{t}-Q_{t}^{\star}\|_{\infty}\leq \varepsilon$.
The consistency error of a policy in~\eqref{eq:def_epsilon} quantifies how the {Q}-value advice is close to optimal {Q}-value functions. It depends on various factors such the function approximation error or training error due to the distribution shift, and has a close connection to a rich literature on value function approximation~\cite{bertsekas1995neuro,munos2003error, antos2008learning, geist2019theory,agarwal2021theory}. 
The results in \cite{antos2008learning} generalized the worst-case $L_{\infty}$ guarantees to arbitrary $L_{p,\rho}$-norms under some mixing assumptions via policy iteration for a stationary Markov decision process (MDP) with a continuous state space and a discrete action space.
Recently, approximation
guarantees for the average case for parametric policy classes (such as a neural network) of value functions have started to appear ~\cite{bertsekas1995neuro,munos2003error,geist2019theory}.  These bounds are useful in lots of supervised machine learning methods such as classification and regression,
whose bounds are typically given on the expected error under some distribution. These results exemplify richer instances of the consistency definition (see~\eqref{eq:def_epsilon}) and a summary of these bounds can be found  in~\cite{agarwal2021theory}.

Now, we are ready to introduce our definition of consistency and robustness with respect to the ratio of expectations, similar to the growing literature on learning-augmented algorithms~\cite{purohit2018improving,wei2020optimal,banerjee2020improving,christianson2022chasing}. We write the ratio of expectations $\mathsf{RoE}(\varepsilon)$ of a policy $\pi$ as a function of the {Q}-value advice error $\varepsilon$ in terms of the $L_\infty$ norm, defined in~\eqref{eq:def_epsilon}.
\begin{definition}[Consistency and Robustness]
\label{def:tradeoff}
An algorithm $\pi$ is said to be \textbf{$k$-consistent} if its worst-case (with respect to the MDP model $\mathsf{MDP}(\mathcal{X},\mathcal{U},\nt,{P},c)$) ratio of expectations satisfies $\mathsf{RoE}(\varepsilon)\leq k$ for $\varepsilon=0$. On the other hand, it is \textbf{$l$-robust} if $\mathsf{RoE}(\varepsilon)\leq l$ for any $\varepsilon>0$.
\end{definition}

\section{The Projection Pursuit Policy (\ouralg)}
\label{sec:p3}

In this section we introduce our proposed algorithm (Algorithm~\ref{alg:ppp}), which
achieves near-optimal consistency while bounding the robustness by
leveraging a robust baseline  (Section~\ref{sec:wasserstein-robust}) in combination with value-based advice (Section~\ref{sec:q-value_advice}).
 A key challenge in the design is how to exploit the benefits of good value-based advice while avoiding following it too closely when it performs poorly. To address this challenge, we propose
to judiciously project the value-based advice into a neighborhood of the robust baseline. By doing so, the actions we choose can follow the value-based advice for consistency while staying close to the robust baseline for robustness.
 More specifically,
at each step $t\in [\nt]$, we choose $u_t=\mathrm{Proj}_{\overline{\mathcal{U}}_t}\left( \widetilde{u}_t\right)$
where a projection operator $\mathrm{Proj}_{\overline{\mathcal{U}}_t}(\cdot):\mathcal{U}\rightarrow\mathcal{U}$ is defined as
\begin{equation}
\label{eq:projection_definition}
\mathrm{Proj}_{\overline{\mathcal{U}}_t}(u) \coloneqq \argmin_{v\in\mathcal{U}} \|u-v\|_{\mathcal{U}} \ \text{subject to } \left\|v-\overline{\pi}_t\left(x_t\right)\right\|_{\mathcal{U}}\leq R_t,
\end{equation}
corresponding to the projection of $u$ onto a ball  $\overline{\mathcal{U}}_t\coloneqq \left\{u\in\mathcal{U}:\left\|u-\overline{\pi}_t\left(x_t\right)\right\|_{\mathcal{U}}\leq R_t\right\}$. Note that when the optimal solution of~\eqref{eq:projection_definition} is not unique, we choose the one on the same line with $\overline{\pi}_t\left(x_t\right)-u$.

The PROjection Pursuit policy, abbreviated as \ouralg, 
can be described as follows. For a time step $t\in [\nt]$, let $\widetilde{\pi}_t:\mathcal{X}\rightarrow\mathcal{U}$ denote a policy that chooses an action $\widetilde{u}_t$ (arbitrarily choose one if there are multiple minimizers of $\smash{\widetilde{Q}_{t}}$), given the current system state $x_t$ at time $t\in [\nt]$ and step $t\in [\nt]$. An action $u_t=\mathrm{Proj}_{\overline{\mathcal{U}}_t}\left(\widetilde{u}_t(x_t)\right)$ is selected by projecting the machine-learned action $\widetilde{u}_t(x_t)$ onto a norm ball $\overline{\mathcal{U}}_t$ defined by the robust policy $\overline{\pi}$ given a radius $R_t\geq0$. Finally, \ouralg applies to both black-box and grey-box settings (which differ from each other in terms of how the radius $R_t$ is decided). The results under
both settings are provided in Section~\ref{sec:main_result}, revealing
a tradeoff between consistency and robustness.

\begin{algorithm}[t]
\SetAlgoLined
\SetKwInOut{Input}{Initialize}\SetKwInOut{Output}{output}
\Input{Untrusted policy $\widetilde{\pi}=\left(\widetilde{\pi}_t:t\in [\nt]\right)$ and baseline policy $\overline{\pi}=\left(\overline{\pi}_t:t\in [\nt]\right)$}
\BlankLine
\For{$t=0,\ldots,\nt-1$}{

{\color{gray}  \textit{//Implement black-box (Section~\ref{sec:black_implementation}) or grey-box (Section~\ref{sec:gray_implementation}) procedures}}

$(\widetilde{u}_t,R_t)\leftarrow$\textsc{Black-Box}$(x_t)$ or $(\widetilde{u}_t,R_t)\leftarrow $\textsc{Grey-Box}$(x_t)$

Set action $u_{t}=\mathrm{Proj}_{\overline{\mathcal{U}}_t}\left( \widetilde{u}_t\right)$ 
where $\overline{\mathcal{U}}_t\coloneqq \left\{u\in\mathcal{U}:\left\|u-\overline{\pi}_t\left(x_t\right)\right\|_{\mathcal{U}}\leq R_t\right\}$

Sample next state $x_{t+1} \sim P_t\left(\cdot|x_t,u_t\right)$

}

\caption{\textbf{\textsc{PRO}}jection \textbf{\textsc{P}}ursuit Policy (\textbf{\ouralg})}
\label{alg:ppp}
\end{algorithm}


The radii $(R_t:t\in [\nt])$ can be interpreted
as  \emph{robustness budgets} and are key design parameters 
that determine the consistency and robustness tradeoff. Intuitively,
the robustness budgets reflect 
the trustworthiness on the value-based policy $\widetilde{\pi}$ --- the larger budgets, the more trustworthiness 
 and hence the more freedom for \ouralg to follow $\widetilde{\pi}$. 
How the robustness budget is chosen
differentiates the grey-box setting from the black-box one. 



\subsection{Black-Box Setting}
\label{sec:black_implementation}

In the black-box setting, 
the only information provided by
$\widetilde{\pi}$ is a suggested action $\widetilde{u}$ for the learning-augmented algorithm. Meanwhile, the robust policy $\overline{\pi}$ can also
be queried to provide advice $\overline{u}$.
Thus, without additional information, a natural way to utilize
both $\widetilde{\pi}$ and $\overline{\pi}$ is to decide a projection radius at each time based on the how the obtained $\widetilde{u}$ and $\overline{u}$. More concretely, at each time $t\in [\nt]$,
the robustness budget $R_t$ is chosen by
the following \textsc{Black-Box} 
Procedure, where we set $R_{t} = \lambda \eta_t$ with
$\eta_t\coloneqq \left\|\widetilde{u}_t-\overline{u}_t\right\|_{\mathcal{U}}$ representing the difference between
the two advice measured in terms of the norm $\|\cdot\|_{\mathcal{U}}$
and $0\leq\lambda\leq 1$ being a tradeoff hyper-parameter that measures the trustworthiness on the machine-learned advice.
The choice of $R_{t} = \lambda \eta_t$ can be explained as follows.
The value of $\eta_t$ indicates the intrinsic discrepancy between the
robust advice and the machine-learned untrusted advice --- the larger discrepancy,
the more difficult to achieve good consistency and robustness simultaneously. 
Given a robust policy and an untrusted policy, by setting a larger $\lambda$, we allow the actual action to deviate more from the robust advice and to follow the untrusted advice more closely, and vice versa. $\lambda$ is a crucial hyper-parameter that can be pre-determined to yield a desired consistency and robustness tradeoff. The computation of $R_t$ is summarized in Procedure~\ref{prodecure:black-box} below.



\setcounter{algorithm}{0}
\begin{myprocedure}[h]



\vspace{1pt}

Implement $\widetilde{\pi}_t$ 
and $\overline{\pi}_t$ to obtain $\widetilde{u}_t$
and $\overline{u}_t$, respectively.

\vspace{1pt}


\vspace{1pt}

Set robustness budget 
 $R_{t} = \lambda \eta_t$  where $\eta_t\coloneqq \left\|\widetilde{u}_t-\overline{u}_t\right\|_{\mathcal{U}}$; Return $(\widetilde{u}_t,R_t)$

\caption{\textsc{Black-Box} 
Procedure at $t\in [\nt]$ (Input: state $x_t$ and hyper-parameter $0\leq \lambda\leq 1$)}

\label{prodecure:black-box}
\end{myprocedure}

\subsection{Grey-Box Setting}
\label{sec:gray_implementation}

In the grey-box setting, along with the suggested action $\widetilde{u}$,
the value-based untrusted policy $\widetilde{\pi}$ also provides an estimate of the {Q}-value function $\widetilde{Q}$ that indicates the long-term cost impact of an action. To utilize such additional information informed
by $\widetilde{Q}_t$ at each time $t\in [\nt]$, we propose a novel algorithm that dynamically adjusts
the budget $R_t$ to further improve the consistency and robustness tradeoff.
More concretely, let us
consider the Temporal-Difference (TD) error $\smash{\mathsf{TD}_t = c_{t-1} + \mathbb{P}_{t-1} \widetilde{V}_{t} -\widetilde{Q}_{t-1}}$. Intuitively, if positive TD error accumulates, the budget $R_t$ should decrease to limit the impact of the learning error. However, the exact TD-error is difficult to compute in practice, since it requires complete knowledge of the transition kernels $(P_t:t\in [\nt])$. 
To address this challenge, we use the following estimated TD-error based on previous trajectories:
\begin{align}
\label{eq:td_error_approximate}
\delta_t\left(x_{t},x_{t-1},u_{t-1}\right) \coloneqq & c_{t-1}\left(x_{t-1},u_{t-1}\right) + \inf_{v\in\mathcal{U}}\widetilde{Q}_{t}\left(x_{t},v\right) - \widetilde{Q}_{t-1}\left(x_{t-1},u_{t-1}\right).
\end{align}
The original construction accumulates this sampled error. In a stochastic MDP, we additionally account for the Q-value gap created by projection and for transition noise. Let $\widetilde V_t(x)\coloneqq\inf_{v\in\mathcal U}\widetilde Q_t(x,v)$ and define
\begin{align}
g_{t-1}&\coloneqq \widetilde Q_{t-1}(x_{t-1},u_{t-1})-\widetilde V_{t-1}(x_{t-1}),
\widehat\delta_t&\coloneqq\delta_t+g_{t-1}\notag\\
&=c_{t-1}(x_{t-1},u_{t-1})+\widetilde V_t(x_t)-\widetilde V_{t-1}(x_{t-1}).
\label{eq:corrected_td}
\end{align}
Let $b_{\nt,t}$ be the deterministic confidence boundary specified in Appendix~\ref{appendix:thm_grey_consistency}. Set $G_0=Z_0=0$ and, for $t\geq1$, define
\begin{equation}
\label{eq:grey_evidence}
G_t\coloneqq\sum_{s=1}^t\widehat\delta_s,
\qquad
Z_t\coloneqq\max_{0\leq s\leq t}[G_s-b_{\nt,s}]^+,
\end{equation}
where $b_{\nt,0}=0$. Denote by $\beta>0$ a hyper-parameter. The \textit{robustness budget} in Algorithm~\ref{alg:ppp} retains the original projection form
\begin{align}
\label{eq:budget}
R_t\coloneqq\Bigg[\underbrace{\left\|\widetilde{\pi}_t\left(x_t\right)-\overline{\pi}_t\left(x_t\right)\right\|_{\mathcal{U}}}_{\color{gray}\text{Decision Discrepancy } \eta_t}-\frac{\beta}{L_Q}Z_t\Bigg]^+.
\end{align}
Equation~\eqref{eq:budget} constitutes two terms. The first term measures the decision discrepancy between the untrusted and baseline policies, which normalizes the total budget as in the black-box setting. The second term uses $Z_t$ from~\eqref{eq:grey_evidence}, normalized by $L_Q$. The boundary removes fluctuations attributable to transition noise, while the running maximum makes the decrease permanent. Since $\eta_t\leq\gamma$, once $Z_t\geq L_Q\gamma/\beta$, the budget is zero at that and every later time. With these terms defined, the \textsc{Grey-Box} Procedure below first chooses a suggested action by minimizing $\widetilde Q_t$ and then determines $R_t$ using~\eqref{eq:budget}.


\begin{myprocedure}[h]

Obtain advice $\widetilde{Q}_{t}$ and $\widetilde{u}_t$ where
$\widetilde{u}_t\in \arginf_{v\in\mathcal{U}}\widetilde{Q}_{t}\left(x_t,v\right)$ 

\vspace{1pt}

Update $G_t$ and $Z_t$ as~\eqref{eq:corrected_td}--\eqref{eq:grey_evidence}

\vspace{1pt}

Implement $\overline{\pi}_t$ and obtain $\overline{u}_t$

\vspace{1pt}

Set robustness budget 
 $R_t$ as~\eqref{eq:budget}; 
Return $(\widetilde{u}_t,R_t)$
\caption{\textsc{Grey-Box} 
Procedure at $t\in [\nt]$ (Input: state $x_t$, hyper-parameter $\beta>0$, and confidence boundary $b_{\nt,t}$)}
\end{myprocedure}

\vspace{-10pt}

\section{Main Results}
\label{sec:main_result}

We now formally present the main results for both the black-box and grey-box settings.
Our results not only quantify the tradeoffs between consistency and robustness formally stated in Definition~\ref{def:tradeoff} with respect to the ratio of expectations, but also emphasize a crucial role that additional information about the estimated {Q}-values plays toward improving the consistency and robustness tradeoff.



\subsection{Black-Box Setting}


In the existing learning-augmented algorithms, the untrusted machine-learned policy $\widetilde{\pi}$ is often treated as a black-box that generates action advice $\widetilde{u}_t$ at each time $t\in [\nt]$. 
Our first result is the following general dynamic regret bound for the black-box setting (Section~\ref{sec:black_implementation}). We utilize the Big-O notation, denoted as $\mathcal{O}(\cdot)$ and $o(\cdot)$ to disregard inessential constants.

\begin{theorem}
\label{thm:black_dr}
Suppose the machine-learned policy $\widetilde{\pi}$ is $(\infty,\varepsilon)$-consistent. For any MDP model satisfying Assumption~\ref{ass:Lip},\ref{ass:robustness}, and~\ref{ass:Q-Lip}, the expected dynamic regret of \ouralg with the \textsc{Black-Box} Procedure is bounded by $\mathsf{DR}(\ouralg)\leq\min\{\mathcal{O}(\varepsilon)+ \mathcal{O}((1-\lambda)\gamma\nt),
    \mathcal{O}\left(\left(\mathsf{ROB}+ \lambda\gamma-1\right)\nt\right)\}$
where $\varepsilon$ is defined in~\eqref{eq:def_epsilon}, $\gamma$ is the diameter of the action space $\mathcal{U}$, $\nt$ is the length of the time horizon, $\mathsf{ROB}$ is the ratio of expectations of the robust baseline $\overline{\pi}$, and $0\leq\lambda\leq 1$ is a hyper-parameter.
\end{theorem}
When $\lambda$ increases, the actual action can deviate more from
the robust policy, making the dynamic regret potentially closer to 
that of the value-based policy.
While the regret bound in Theorem~\ref{thm:black_dr} clearly shows the role of $\lambda$ 
in terms of controlling how closely we follow the robust policy, 
the dynamic regret given a fixed $\lambda\in [0,1]$ grows linearly 
in $\mathcal{O}(\nt)$. In fact, the linear growth of dynamic regret holds even
if the black-box policy $\widetilde{\pi}$ is consistent, i.e., $\varepsilon$ is small. This can be explained by noting the lack of
dynamically tuning $\lambda$ to follow the better of the two policies --- even when one policy is nearly perfect, the actual action still always deviates from it due to the fixed choice of $\lambda$.

Consider any MDP model satisfying Assumptions~\ref{ass:Lip},\ref{ass:robustness}, and~\ref{ass:Q-Lip}. Following the classic definitions of consistency and robustness (see Definition~\ref{def:tradeoff}), we summarize the following characterization of \ouralg, together with a negative result in Theorem~\ref{thm:black_impossibility}. Proofs of Theorem~\ref{thm:black_dr}, \ref{thm:black_consistency_robustness}, and \ref{thm:black_impossibility} are detailed in Appendix~\ref{app:black_box}.

\begin{theorem}[\textsc{Black-Box} Consistency and Robustness]
\label{thm:black_consistency_robustness}
\ouralg with the \textsc{Black-Box} Procedure is $\left(1+\mathcal{O}((1-\lambda)\gamma)\right)$-consistent and $(\mathsf{ROB}+\mathcal{O}(\lambda\gamma))$-robust where $0\leq\lambda\leq 1$ is a hyper-parameter.
\end{theorem}

\begin{theorem}[\textsc{Black-Box} Impossibility]
\label{thm:black_impossibility}
\ouralg with the \textsc{Black-Box} Procedure cannot be both $\left(1+o((1-\lambda)\gamma)\right)$-consistent and $(\mathsf{ROB}+o(\lambda\gamma))$-robust for any $0\leq\lambda\leq 1$.
\end{theorem}

\subsection{Grey-Box Setting}


To overcome the impossibility result in the black-box setting, we dynamically tune the robustness budgets by tapping into additional information from the estimated {Q}-value functions using the \textsc{Grey-Box} Procedure (Section~\ref{sec:gray_implementation}). An analogous dynamic regret bound is given in Appendix~\ref{app:grey_box}. Consider any MDP model satisfying Assumptions~\ref{ass:Lip}, \ref{ass:robustness}, and~\ref{ass:Q-Lip}. Our main result below shows that structural information about an untrusted black-box policy can improve the consistency and robustness tradeoff.

\begin{theorem}[\textsc{Grey-Box} Consistency and Robustness]
\label{thm:grey_robustness_consistency}
Under the class-uniform conditions stated in Appendix~\ref{appendix:thm_grey_consistency}, \ouralg with the \textsc{Grey-Box} Procedure is $(1+o(1))$-consistent under exact Q-value advice and $(\mathsf{ROB}+o(1))$-robust over the admissible advice class for every fixed $\beta>0$. More generally, the same conclusion holds for every horizon-dependent $\beta_{\nt}>0$ satisfying $1/\beta_{\nt}=o(\nt)$. For deterministic transitions, exact Q-value advice gives $1$-consistency at every horizon.
\end{theorem}

Theorem~\ref{thm:black_impossibility} shows that the \textsc{Black-Box} Procedure cannot attain the same simultaneous asymptotic consistency and robustness limits, while the \textsc{Grey-Box} Procedure can do so under the stated conditions. On the one hand, this result shows the effectiveness of \ouralg with value-based machine-learned advice that may not be fully trusted. On the other hand, the contrast between the black-box and grey-box settings reveals that access to value information can improve the consistency and robustness tradeoff non-trivially.

The stochastic result additionally uses the sublinear Q-value error, a time uniform confidence bound for transition noise with vanishing failure contribution, and a baseline guarantee that remains valid after the permanent switch. The precise conditions, finite horizon factors, and proof are given in Appendix~\ref{appendix:thm_grey_consistency}. Applications of our main results are discussed in Appendix~\ref{sec:application}.

\section{Concluding Remarks}

Our results contribute to the growing body of literature on learning-augmented algorithms for MDPs and highlight the importance of considering consistency and robustness in this context. In particular, we have shown that by utilizing the \textit{structural information} of machine learning methods, it is possible to achieve improved performance over a black-box approach. The results demonstrate the potential benefits of utilizing value-based policies as advice; however, there remains room for future work in exploring other forms of structural information.

\textbf{Limitations and Future Work. }One limitation of our current work is the lack of analysis of more general forms of black-box procedures. Understanding and quantifying the available structural information in a more systematic way is another future direction that could lead to advances in the design of learning-augmented online algorithms and their applications in various domains.  

\section*{Acknowledgement}

We would like to thank the anonymous reviewers for their helpful comments.
This work was supported in part by the National Natural Science Foundation of China (NSFC) under grant No. 72301234, the Guangdong Key Lab of Mathematical Foundations for Artificial Intelligence, and the start-up funding UDF01002773 of CUHK-Shenzhen. Yiheng Lin was supported by the Caltech Kortschak Scholars program. Shaolei Ren was supported in part by the U.S. U.S. National Science Foundation (NSF) under grant CNS--1910208. 
Adam Wierman was supported in part by the U.S. NSF under grants CNS--2146814, CPS--2136197, CNS--2106403, NGSDI--2105648.

\addcontentsline{toc}{section}{Bibliography}
\bibliographystyle{unsrt}
{\bibliography{main}}

\newpage

\appendix

\section{Application Examples}
\label{sec:application}

In this section, we delve deeper into the practical applications of our main results, which provide a general consistency and robustness tradeoff. By presenting concrete examples, we aim to demonstrate the versatility and relevance of our findings to various real-world problems and scenarios. We consider the settings summarized in Table~\ref{table:baseline}. The analytical examples verify model and baseline components of the framework, while the later numerical studies use the separate empirical heuristic stated below. Additionally, these examples highlight the significance of considering the tradeoff between consistency and robustness in the design and implementation of decision-making algorithms in the learning-augmented framework, and the impact of the structural information in the grey-box setting.

\subsection{MPC baseline in time-varying dynamical systems}
\label{sec:mpc_baseline}

The first application is an online optimal control problem, which is a special case of the general MDP in Section \ref{sec:model}. Suppose that the dynamics and cost function in time $t \in [\nt]$ are given by 
\begin{align}
    \label{eq:canonical}
    x_{t+1} = A_t x_t + B_t u_t + w_t
\end{align}
and 
\begin{align}
\label{eq:quadratic_costs}
     c_t(x_t, u_t) = \frac{1}{2} \left((x_t)^\top Q_t x_t + (u_t)^\top R_t u_t\right).
\end{align}
Here, $\left(w_t:t\in [\nt]\right)$ is a sequence of bounded and oblivious disturbances that is unknown to the online controller. \footnote{By oblivious, we mean the sequence $\left(w_t:t\in [\nt]\right)$ is determined by the environment before the game starts and are not random.} At each time step, the controller observes $(A_t, B_t, Q_t, R_t)$ for future $k$ time steps but all future disturbances are unknown. Since we assume $c_{\nt} \equiv 0$, the optimal $u_{\nt-1}$ is always $0$ and the online control problem in episode $t$ actually terminates after the state $x_{\nt-1}$ is revealed. 

We show how to apply Model Predictive Control (MPC) with robust predictions as the robust baseline in our framework \cite{li2022robustness,lin2022online}. To define MPC with robust predictions, we first need to define the finite-time optimal control problem (FTOCP) solved by MPC at every step: For $t, t' \in [\nt]$, we define
\begin{align*}
\psi_{t, t'}\left(x_t, \left(w_{\tau\mid t}:\tau \in [t:t'-1]\right); P_{t'}\right) = &\argmin_{u_{t:(t'-1)\mid t}} \sum_{\tau = t}^{t' - 1} c_\tau\left(x_{\tau\mid t}, u_{\tau\mid t}\right) + \frac{1}{2} x_{t'\mid t}^\top P_{t'} x_{t'\mid t}\\
\text{s.t. }& x_{\tau+1\mid t} = A_{t} x_{\tau\mid t} + B_{t} u_{\tau\mid t} + w_{\tau\mid t}, \forall \tau \in [t:t'-1];\\
& x_{t\mid t} = x_t,
\end{align*}
Here, $\left(w_{\tau\mid t}:{\tau \in [t:t'-1]}\right)$ can be viewed as the predicted future disturbances, and MPC with robust predictions sets them to zero vectors, therefore becomes robust against (potentially) adversarial environments with large disturbances $\left(w_t:t\in [nt]\right)$. The term $ x_{t'\mid t}^\top P_{t'} x_{t'\mid t}/2$ is a terminal cost that regularizes the last predictive state. To simplify the notation, we use the shorthand $\psi_{t, t'}(x_t; P) \coloneqq \psi_{t, t'}(x_t, 0_{\times (t' - t)}; P)$, where $0_{\times (t' - t)}$ denotes a sequence of $(t'-t)$ zeros, and
\[\psi_{t, t'}(x_t) \coloneqq \begin{cases}
\psi_{t, t'}(x_t; P_{t'}) &\text{ if } t' < \nt-1,\\
\psi_{t, t'}(x_t; Q_{t'}) &\text{ if } t' = \nt-1.
\end{cases}\]
With this notation, we define MPC with robust predictions formally in Algorithm \ref{alg:mpc-baseline}. At each time step, it solves a $k$-step predictive FTOCP and commits the first control action in the optimal solution. Since the future disturbances are unknown, MPC predicts them to zero vectors. The terminal cost matrices $P_t$ are pre-determined.

\begin{algorithm}[t]
\SetAlgoLined
\SetKwInOut{Input}{Initialize}\SetKwInOut{Output}{output}
\Input{Prediction horizon $k$ and terminal cost matrix $P$.}
\BlankLine
\For{$t=0,\ldots,\nt-1$}{

Set $t' \gets \min\{t + k, \nt-1\}$

Observe $x_t$ and $\left(A_{\tau}, B_{\tau}, Q_{\tau}, R_{\tau}:{\tau \in [t:t'-1]}\right)$



Set action $u_t = \psi_{t, t'}(x_t)[u_{t\mid t}]$
}

\caption{MPC with Robust Predictions ($\MPC_k$)}
\label{alg:mpc-baseline}
\end{algorithm}

We make some standard assumptions, following those in the literature of online control \cite{lin2021perturbation,zhang2021regret,lin2022bounded}.
The first assumption is that the cost functions are well-conditioned and the dynamical matrices are uniformly bounded.
\begin{assumption}\label{assump:bounded-costs-and-dynamics}
For any $t \in [\nt]$, we have $\norm{A_t} \leq a, \norm{B_t} \leq b, \norm{w_t} \leq d$, and
\[\mu I_n \preceq Q_t \preceq \ell I_n, \mu I_m \preceq R_t \preceq \ell I_m, \mu I_n \preceq P \preceq \ell I_n.\]
\end{assumption}

The second assumption guarantees that for arbitrary bounded disturbance sequences $\left(w_t:t\in [\nt]\right)$, there exists a controller that can stabilize the system.
\begin{assumption}\label{assump:uniform-stability}
For any $t, t' \in [\nt], t \leq t'$, define a block matrix
\begin{align*}
\Xi_{t, t'}^t \coloneqq \begin{bmatrix}
    I & & & & \\
    - A_t & - B_t & I & & \\
     & & \ddots & & & & \\
     & & - A_{t'-1} & - B_{t'-1} & I
    \end{bmatrix}.
\end{align*}
We assume $\sigma_{\min}\left(\Xi_{t, t'}\right) \geq \sigma$ for some positive constant $\sigma$, where $\sigma_{\min}(\cdot)$ denotes the smallest singular value of a matrix.
\end{assumption}

The interpretation of Assumption \ref{assump:uniform-stability} is as follows. It holds with $\sigma$ if and only if for any sequence $x_t, w_{t:t'-1}$ that satisfies $\norm{\left((x_t)^\top, (w_t)^\top, \ldots, (w_{t'-1})^\top\right)} \leq 1$, there exists a feasible trajectory $x_t, u_t, x_{t+1}, \ldots, u_{t'-1}, x_{t'}$ subject to \[\norm{\left((x_t)^\top, (u_t)^\top, (x_{t+1})^\top, \ldots, (u_{t'-1})^\top, (x_{t'})^\top\right)}\leq \frac{1}{\sigma}.\] 
Thus, Assumption~\ref{assump:uniform-stability} holds provided that there is an exponentially stabilizing controller. With these assumptions, we are ready to present our main result about $\MPC_k$ in Theorem \ref{thm:W-Robustness-and-ROB}.

\begin{theorem}\label{thm:W-Robustness-and-ROB}
Suppose Assumptions \ref{assump:bounded-costs-and-dynamics} and \ref{assump:uniform-stability} hold. Consider the case when the robust baseline policy $\overline{\pi}$ is $\MPC_k$ (Algorithm \ref{alg:mpc-baseline}), and for some $\overline{\decayfactor} \in (\decayfactor, 1)$, the prediction horizon $k$ satisfies that $$k \geq \min\left\{\nt, \frac{1}{2}\log\left(C^3 b a \decayfactor / (\overline{\decayfactor} - \decayfactor)\right)/\log(1/\decayfactor)\right\}.$$ We also assume that $x_0 = 0$, and $R_t \leq \overline{R}$. Then, the following holds for the robust baseline $\overline{\pi}$:
\begin{enumerate}[(i)]
    \item The Wasserstein robustness (Definition \ref{def:robust}) holds globally with $s(t) = C (1 + C)(a + b) \overline{\decayfactor}^{t-1}$.
    \item The \ouralg controller (Algorithm \ref{alg:ppp}) is always stable in the sense that
    \begin{align*}
        \norm{x_t} \leq \overline{R}_x \coloneqq \frac{C(d + b \overline{R})}{1 - \overline{\lambda}} \text{ and } \norm{u_t} \leq \overline{R}_u \coloneqq C \overline{R}_x + \overline{R}.
    \end{align*}
    \item The competitive ratio of the robust baseline $\MPC_k$ satisfies that
    \[\mathsf{ROB} \leq \frac{2\ell C^2(1 + C^2)(1 + a^2 + b^2)}{\mu (1 - \overline{\decayfactor})^2}.\]
\end{enumerate}
Here, the coefficients $C$ and $\lambda$ are given by
$\decayfactor = \left(\frac{\overline{\sigma} - \underline{\sigma}}{\overline{\sigma} + \underline{\sigma}}\right)^{\frac{1}{2}}, C = \frac{4(\ell + 1 + a + b)}{\underline{\sigma}^2 \cdot \decayfactor},$
where
\[\underline{\sigma} \coloneqq \min(\mu, 1) \cdot (a + b + 1) \cdot \sqrt{\frac{\ell}{2\mu \ell + \mu \sigma^2}}, \text{ and } \overline{\sigma} \coloneqq \sqrt{2} (\ell + a + b + 1).\]
\end{theorem}

The first result of Theorem \ref{thm:W-Robustness-and-ROB} shows that $\MPC_k$ satisfies the Wasserstein robustness (see Definition~\ref{def:robust}), which is the critical assumption we require for any robust baseline policy. The second result guarantees that \ouralg (with $\MPC_k$ as the robust baseline) will always stay in a bounded ball in the Euclidean space as long as the radius $R_t$ is uniformly upper bounded. Thus, we can assume $\mathcal{X}$ and $\mathcal{U}$ are compact without loss of generality. The third result gives an upper bound of the robust competitive ratio $\mathsf{ROB}$, which in this application is a special deterministic case of the considered ratio of expectation ($\mathsf{ROE}$) in the general results. For deterministic transitions, Theorem~\ref{thm:grey_robustness_consistency} gives exact consistency under exact Q-value advice. Its robustness conclusion additionally requires the restartable suffix guarantee in~\eqref{eq:restartable_baseline}. The linear normalization can be imposed by adding the same positive constant to every stage cost, which does not change the MPC or optimal actions.
We defer the detailed proof of Theorem \ref{thm:W-Robustness-and-ROB} to Appendix \ref{appendix:MPC-proof}.

\subsection{Baseline Policies for MDPs with Finite State/Action Spaces}
\label{appendix:Wasserstein-TV-distance}
Our second example focuses on an MDP environment $\left(\mathcal{S}, \mathcal{A}, \left(\mathbb{P}_t:t\in [\nt]\right), \left(c_t:t\in [\nt]\right), \nt\right)$ with a finite state space $\mathcal{S}$ and a finite action space $\mathcal{A}$. Given a policy $\overline{\pi}_t: \mathcal{S} \to \Delta(\mathcal{A})$ for $t \in [\nt]$, let $\left(\overline{\mathbb{P}}_t\right)$ denote the state transition probability that maps $\mathcal{S}$ to $\Delta(\mathcal{S})$, which is defined as
\[\overline{\mathbb{P}}_t(s; s') = \sum_{a \in \mathcal{A}} \overline{\pi}_t(s; a) \mathbb{P}_t(s, a; s').\]
We consider the setting when every entry of $\overline{\mathbb{P}}_t$ is strictly positive. Under this assumption, one can show that the one-step transition probability $\overline{\mathbb{P}}_t$ is a contractive mapping in total variance distance \cite{norris1998markov}. We state this result formally in Lemma~\ref{lemma:MDP-exp-contraction}. To simplify the notation, for any $0 \leq t \leq t' < \nt$, we define the multi-step transition matrix as $\overline{\mathbb{P}}_{t:t'} \coloneqq \overline{\mathbb{P}}_t \overline{\mathbb{P}}_{t+1} \cdots \overline{\mathbb{P}}_{t'}$.

\begin{lemma}\label{lemma:MDP-exp-contraction}
Under the assumption that $\min_{t \in [\nt]} \min_{s, s' \in \mathcal{S}} \overline{\mathbb{P}}_t(s; s') \geq \epsilon$, for any $0 \leq t \leq t' < \nt$ and distributions $\mu, \nu \in \Delta(\mathcal{S})$, we have that
\begin{equation}\label{lemma:MDP-exp-contraction:e1}
    \TV\left(\mu^\top \overline{\mathbb{P}}_{t:t'}, \nu^\top \overline{\mathbb{P}}_{t:t'}\right) \leq \decayfactor^{t' - t} \TV(\mu, \nu),
\end{equation}
where $\decayfactor = 1 - \abs{\mathcal{S}}\epsilon$.
\end{lemma}

Lemma~\ref{lemma:MDP-exp-contraction} follows from Proposition 5 in \cite{MarkovChainNote}. In the case that not every entry of $\overline{\mathbb{P}}_t$ is strictly positive, but the entries of $\overline{\mathbb{P}}_{t:t+d}$ are strictly positive for some constant $d \in \mathbb{Z}_+$, we can still obtain a similar contraction property as Lemma~\ref{lemma:MDP-exp-contraction} with a weaker decay rate $\decayfactor$. Note that the exponential contractive property in Lemma~\ref{lemma:MDP-exp-contraction} is different with the one in Wasserstein robustness (Definition~\ref{def:robust}) because the distance between distributions are measured by total variance instead of Wasserstein distance. To convert it into the form required by Wasserstein robustness, we need to define an underlying metric for the discrete state/action space.

Without loss of generality, we assume $\mathcal{X} \coloneqq \left\{e_i:i = 1, \ldots, \abs{S}\right\} \subseteq \mathbb{R}^{\abs{\mathcal{S}}}$, where each element of $\mathcal{X}$ corresponds to a unique state in $\mathcal{S}$. Here, each $e_i$ is an indicator vector of $\mathbb{R}^{\abs{\mathcal{S}}}$ defined as
\begin{align*}
    e_i(j) = \begin{cases}
        1 & \text{ if } i = j,\\
        0 & \text{ otherwise.}
    \end{cases}
\end{align*}
Since a policy in the discrete MDP maps $\mathcal{S}$ to $\Delta(\mathcal{A})$, we set $\mathcal{U} = \Delta(\mathcal{A}) \subseteq \mathbb{R}^{\abs{\mathcal{A}}}$, which denotes the distribution of actions and is compact and convex. To define the Wasserstain distance, we adopt $\ell_1$ distance as the metric on the state space $\mathcal{X}$, action space $\mathcal{U}$, and state-action space $\mathcal{X} \times \mathcal{U}$, i.e.,
\[\norm{(x, u) - (x', u')}_1 = \norm{x - x'}_1 + \norm{u - u'}_1, \text{ for all } x, x' \in \mathcal{X}, u, u' \in \mathcal{U}.\]
Using these definitions, we can use the contraction property in the TV distance (Lemma \ref{lemma:MDP-exp-contraction}) to establish the Wasserstein robustness of the baseline policy $\overline{\pi}$.

\begin{theorem}\label{thm:MDP-baseline-W-Robustness}
Suppose the Markov chain on state space $\mathcal{S}$ induced by the baseline policy $\overline{\pi} = (\overline{\pi}_t: t \in [\nt])$ satisfies that $\overline{\mathbb{P}}_t(s; s') \geq \epsilon$ for all $t \in [\nt]$ and $s, s' \in \mathcal{S}$, then Wasserstein robustness (Definition \ref{def:robust}) holds globally with $s(t) = 2 \decayfactor^{t-1}$, where $\decayfactor = 1 - \abs{\mathcal{S}}\epsilon$.
\end{theorem}

We defer the proof of Theorem~\ref{thm:MDP-baseline-W-Robustness} to Appendix~\ref{appendix:MDP-proof}. Theorem~\ref{thm:MDP-baseline-W-Robustness} shows that the Wasserstein robustness in Definition~\ref{def:robust} is general enough to capture a wide class of baseline policies in finite state/action settings. It also enables comparison between our results and  previous studies that assume discrete state/action spaces \cite{jin2018q,mao2021near,golowich2022can}. 

\subsection{Numerical Results}

In light of the applications detailed in Appendix~\ref{sec:mpc_baseline}, we present two case studies. We consider linear dynamics as a specific instance of an MDP and use the MPC described in Algorithm~\ref{alg:mpc-baseline} as our robust baseline. 

\subsubsection{Basic Settings}

\paragraph{Dynamics.}
We investigate the impact of the hyper-parameter $\beta$ used in the empirical robustness budget described below by considering the following update rule:
\begin{align}
\label{eq:canonical2}
 \begin{bmatrix}
    d_{t+1}\\
    v_{t+1}
    \end{bmatrix} = A \begin{bmatrix}
    d_{t}\\
    v_{t}
    \end{bmatrix} + B u_t + w_t,
\end{align}
which is cast in the canonical form~\eqref{eq:canonical}. The system matrices $A$ and $B$ are defined as
\begin{align}
\label{eq:A_and_B}
A\coloneqq
    \begin{bmatrix}
    1 & 0 & 0.2 & 0\\
    0 & 1 & 0 & 0.2\\
    0 & 0 & 1 & 0 \\
    0 & 0 & 0 & 1
    \end{bmatrix}, \quad
B\coloneqq
    \begin{bmatrix}
    0 & 0 \\
    0 & 0 \\
    0.2 & 0 \\
    0 & 0.2
    \end{bmatrix},
\end{align}
and $w_t 
\coloneqq \begin{bmatrix}
    y_t \\ 0_{2}
\end{bmatrix} - \begin{bmatrix}
    y_{t+1} \\ 0_{2}
\end{bmatrix}$, where $(y_t:t\in [\nt])$ specifies an unknown trajectory to be tracked. 
The choice of $A,B$ and $(w_t:t\in [\nt])$ specifies a two-dimensional robot tracking problem as detailed in \cite{li2019online,li2022robustness}.
In this application, the robot controller maneuvers along a fixed but unknown trajectory, given by $(y_t:t\in [\nt])$. At each time $t\in [\nt]$, the robot controller needs to decide an acceleration action $u_t$, without knowing the desired location $y_t$. It can only access the past trajectories $(y_{\tau}:\tau\in [t])$.
The location of the robot controller at time \( t+1 \), denoted \( l_{t+1}\in\mathbb{R}^2 \), is determined by its prior location and its velocity \( v_t\in\mathbb{R}^2 \) according to \( l_{t+1}=l_t+0.2v_t \). Furthermore, at each subsequent time \( t \), the controller has the ability to apply an adjustment \( u_t \) to alter its velocity, resulting in \( v_{t+1}=v_t+0.2u_t \) at the next time step. This system can be reformulated as \eqref{eq:canonical2} by letting $x_t=l_t-y_t$, the tracking error between the current location at time $t$ and the desired location $y_t$.

To efficiently track the trajectory, we use quadratic costs as in~\eqref{eq:quadratic_costs} with
\begin{align}
\label{eq:Q_and_R}
Q \coloneqq
    \begin{bmatrix}
    1 & 0 & 0 & 0\\
    0 & 1 & 0 & 0\\
    0 & 0 & 0 & 0 \\
    0 & 0 & 0 & 0
    \end{bmatrix}, \quad
R \coloneqq
    \begin{bmatrix}
    10^{-2} & 0 \\
    0 & 10^{-2} 
    \end{bmatrix}.
\end{align}

With the settings above, we encapsulate a Gym environment~\cite{brockman2016openai} with action space and state space defined as hyper-cubes in $\mathbb{R}^2$ and $\mathbb{R}^4$ such that each action/state coordinate is within $[-100,100]$.

\paragraph{MPC Baseline $\overline{\pi}$.}
With predictions $(\widetilde{w}_t:t\in [\nt])$ of the perturbations satisfy $\widetilde{w}_t=0$ for all $t\in [\nt]$, with a terminal cost matrix $P$ as the  solution of the discrete algebraic Riccati equation (DARE), 
the MPC baseline in Algorithm~\ref{alg:mpc-baseline} can be stated as the following linear quadratic regulator
$ \overline{\pi}_{\MPC}(x_t) = -Kx_t
$
where $K \coloneqq (R+B^\top P B)^{-1} B^\top P A$.

\paragraph{Machine-Learned Policy $\widetilde{\pi}$.} We use the deep deterministic policy gradient (DDPG) algorithm~\cite{silver2014deterministic} to generate machine learned advice and Q-value functions, with hyper-parameters set as in Table~\ref{table:ddpg}.

\begin{table}[ht]
\footnotesize
\renewcommand{\arraystretch}{1.3}
    \centering
       \caption{Hyper-parameters used in DDPG.}
    \begin{tabular}{l|l}
\specialrule{.11em}{.1em}{.1em} 
   \textbf{Parameter}  & \textbf{Value}  \\
   \hline
   Maximal number of episodes & $10^3$ \\
   Episode length & $10^2$ \\
    Discount factor & $1.0$\\
   Actor network learning rate & $10^{-3}$ \\
   Critic network learning rate & $10^{-3}$ \\
   Soft target update parameter & $10^{-3}$\\
   Replay buffer size & $10^6$ \\
   Minibatch size &  $128$\\
\specialrule{.11em}{.1em}{.1em} 
    \end{tabular}
  \label{table:ddpg}
\end{table}

\paragraph{\ouralg Implementation}

In the experiments, we set $L_Q=1$ and use the empirical heuristic
\[
R_t^{\mathrm{emp}}\coloneqq[\eta_t-\beta|\delta_t|]^+,
\]
where $\delta_t$ is defined in~\eqref{eq:td_error_approximate}. These experiments were run using the absolute sampled TD error and do not use the corrected statistic $G_t$, the confidence boundary $b_{\nt,t}$, or the running maximum $Z_t$ in~\eqref{eq:grey_evidence}. The figures below therefore study this empirical heuristic and do not test the guarantee in Theorem~\ref{thm:grey_robustness_consistency}. All reported figures and numerical values are unchanged.

\begin{figure}[h]
\centering
\includegraphics[width=0.4\textwidth]{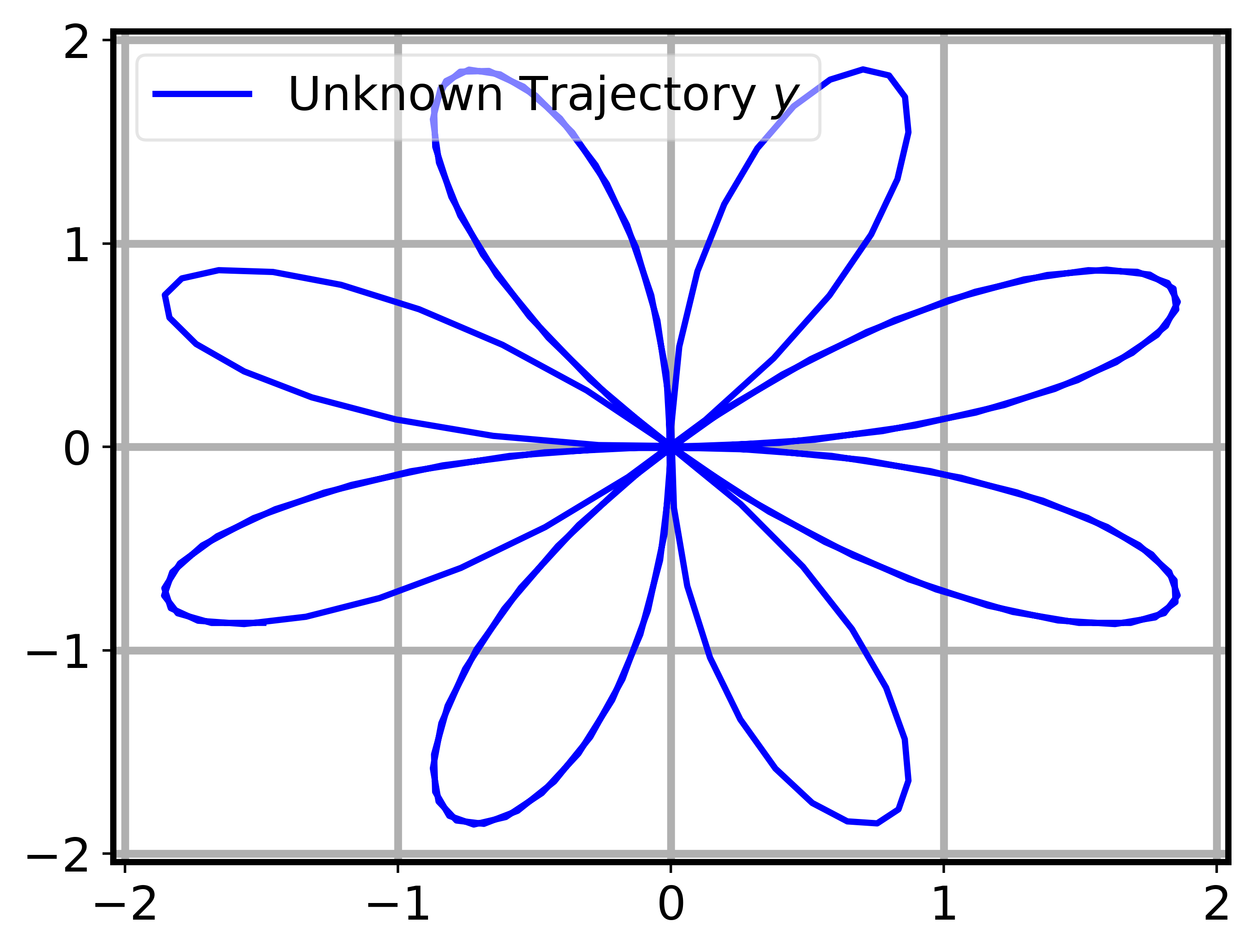}
\caption{Unknown trajectory $y$ used in the case study of the empirical heuristic budget $R_t^{\mathrm{emp}}$.}
\label{fig:trajectory}
\end{figure}

\begin{figure}[h]
\begin{subfigure}[b]{0.495\textwidth}
    \centering
\includegraphics[width=\textwidth]{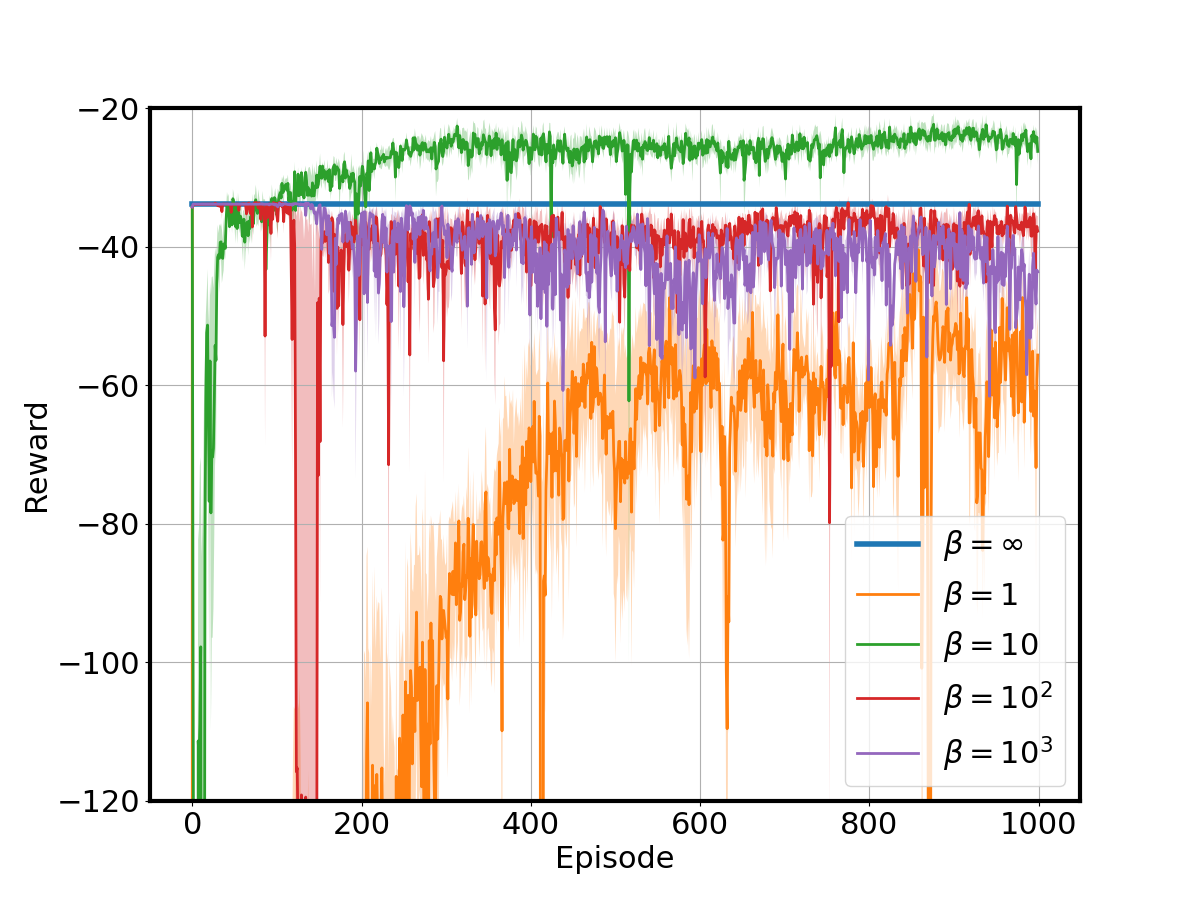}
\end{subfigure}
\hfill
\begin{subfigure}[b]{0.495\textwidth}
    \centering
\includegraphics[width=\textwidth]{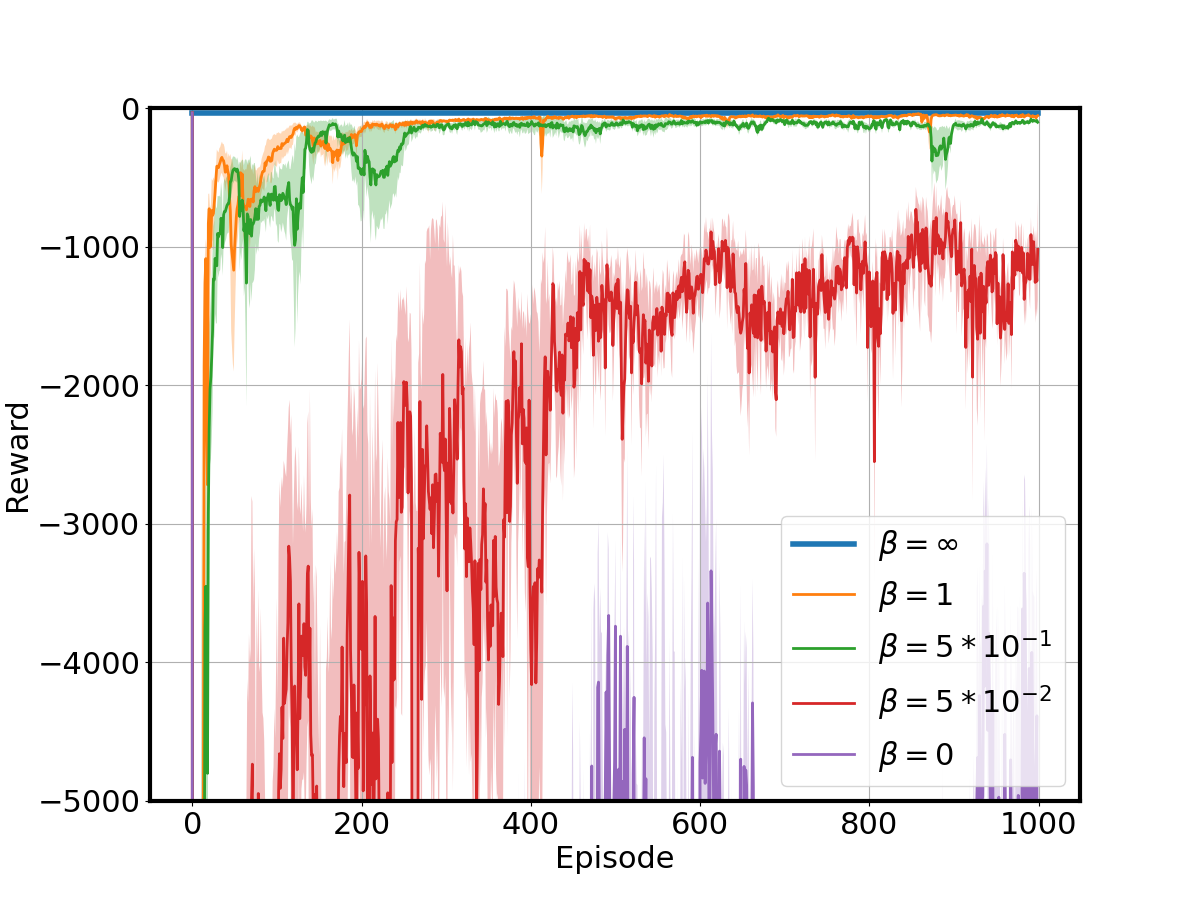}
\end{subfigure}
\caption{Average awards with varying choices of the hyper-parameter $\beta$ in the empirical heuristic budget $R_t^{\mathrm{emp}}$. The shaded area depicts the range of standard deviations for $5$ random tests. The left panel uses $\beta=1,10,10^2,10^3$, and $\infty$, with $\beta=\infty$ implemented by directly applying the MPC baseline. The right panel uses $\beta=0,0.05,0.5,1$, and $\infty$.}
\label{fig:beta}
\end{figure}

\begin{figure}[h]
\centering
\includegraphics[width=0.9\textwidth]{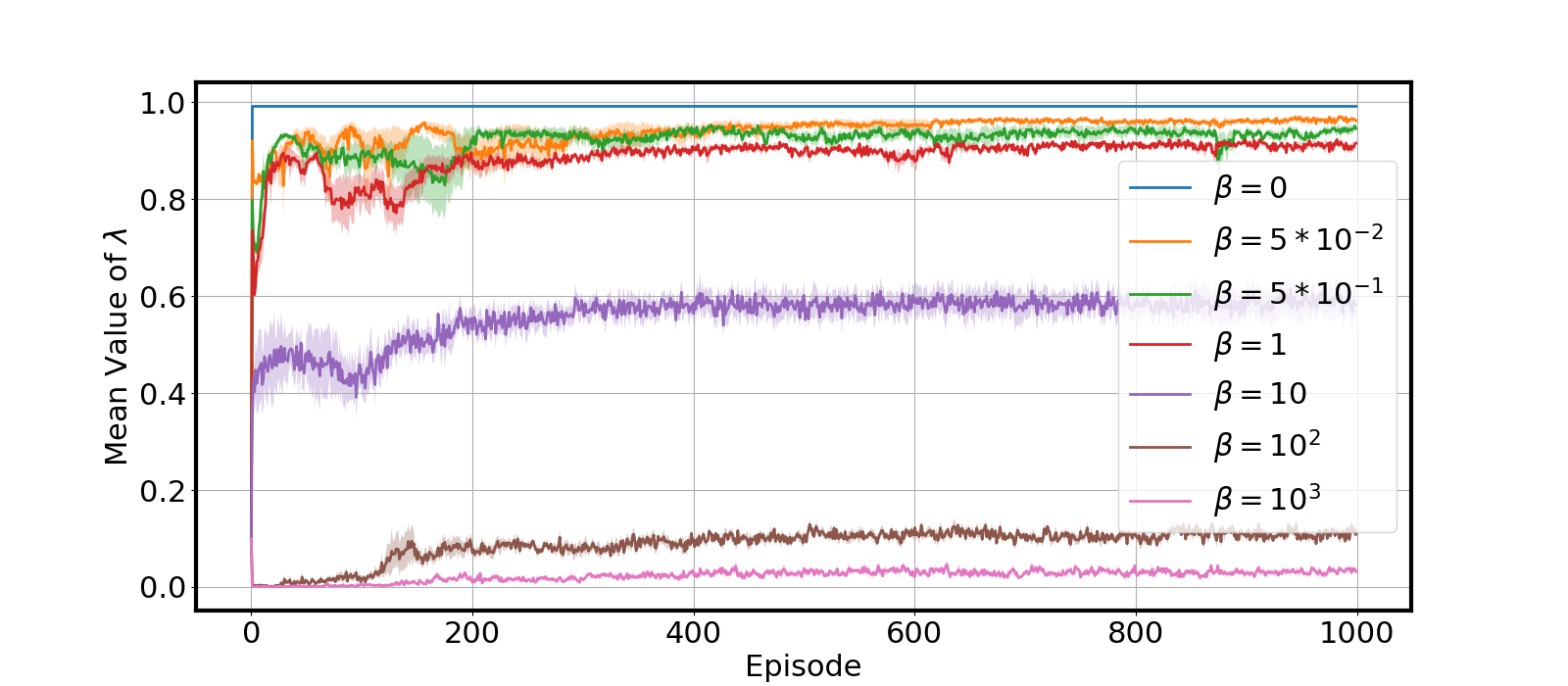}
\caption{The influence of the hyper-parameter \(\beta\) on the projection radii $(R_t^{\mathrm{emp}}:t\in[\nt])$ generated by the empirical heuristic. The shaded region represents the standard deviation range from $5$ random tests. As $\beta$ increases, the average trust coefficient $\lambda(R_t^{\mathrm{emp}})$ decreases.}
\label{fig:lambda}
\end{figure}

\begin{figure}[H]
\centering
\includegraphics[width=\textwidth]{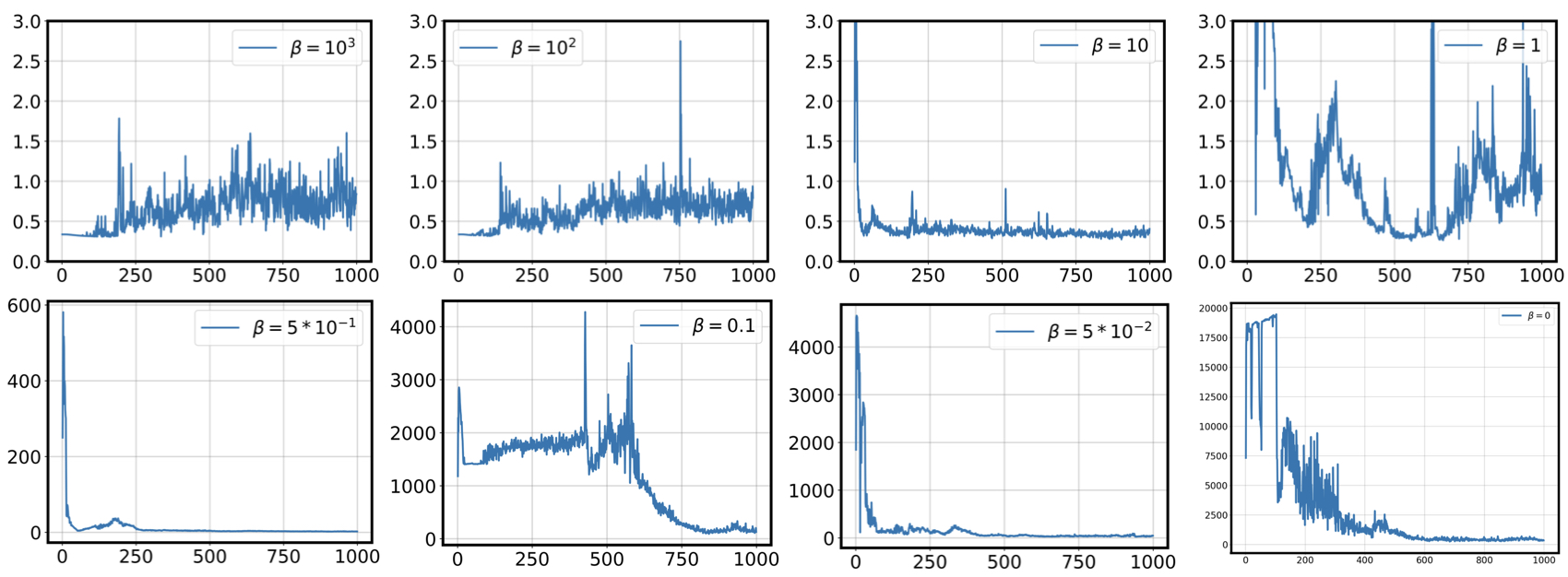}
\caption{Average absolute sampled TD error $|\delta_t|$, defined in~\eqref{eq:td_error_approximate}, with different choices of the hyper-parameter $\beta$ in the empirical heuristic budget $R_t^{\mathrm{emp}}$. The shaded area depicts the range of standard deviations for $5$ random tests.}
\label{fig:td_error}
\end{figure}

\subsubsection{Case Studies}
\label{app:experimental_results}

\paragraph{Impact of Hyper-Parameter $\beta$.} 

With the basic settings described above, we study how the hyper-parameter \(\beta\) affects the empirical heuristic budget $R_t^{\mathrm{emp}}$. This choice influences the average rewards, projection radii, and absolute sampled TD error reported below.

We set the unknown $(y_t:t\in [\nt])$  to be tracked as a rose-shaped trajectory shown in Figure~\ref{fig:trajectory}:
\begin{align}
\label{eq:trajectory}
y_t\coloneqq \begin{bmatrix}
     2\cos\left(\frac{t}{20}\right) \sin\left( \frac{t}{5}\right)\\
    2\sin\left(\frac{t}{20}\right)\sin\left( \frac{t}{5}\right)
    \end{bmatrix}, \quad t\in [\nt].
\end{align}

We vary $\beta$ from \(0\) to \(\infty\). When \(\beta=0\), the empirical implementation is the same as pure DDPG. When \(\beta=\infty\), we directly apply the MPC baseline. Arbitrary exploration in the action space will lead to unstable states, causing the pure DDPG to remain non-convergent throughout its training process. From our experiments, we observe that setting \(\beta\) between \(5\) and \(25\) yields the largest average reward. The results are summarized in Figure~\ref{fig:beta}. Under the projection rule, the action used in the experiments at each time $t\in [\nt]$ can be written as
\begin{align*}
    u_t = \lambda\left(R_t^{\mathrm{emp}}\right)\widetilde{\pi}_t\left(x_t\right)+\left(1-\lambda\left(R_t^{\mathrm{emp}}\right)\right)\overline\pi_t\left(x_t\right)
\end{align*}
where $\lambda\left(R_t^{\mathrm{emp}}\right)\coloneqq \min\left\{1,{R_t^{\mathrm{emp}}}/{\left\|\widetilde{\pi}_t\left(x_t\right)-\overline\pi_t\left(x_t\right)\right\|_{\mathcal{U}}}\right\}$ serves as a \textit{trust coefficient} between $0$ and $1$. Figure~\ref{fig:lambda} illustrates $\lambda(R_t^{\mathrm{emp}})$ averaged over all time steps and tests. Figure~\ref{fig:td_error} displays the average absolute sampled TD error $|\delta_t|$ for different choices of $\beta$. This quantity visibly stabilizes when $\beta=10$, which also yields high average rewards in Figure~\ref{fig:beta}.
It's worth noting, however, that we did not actively optimize for $\beta$.

\begin{figure}[h]
\centering
\includegraphics[width=1\textwidth]{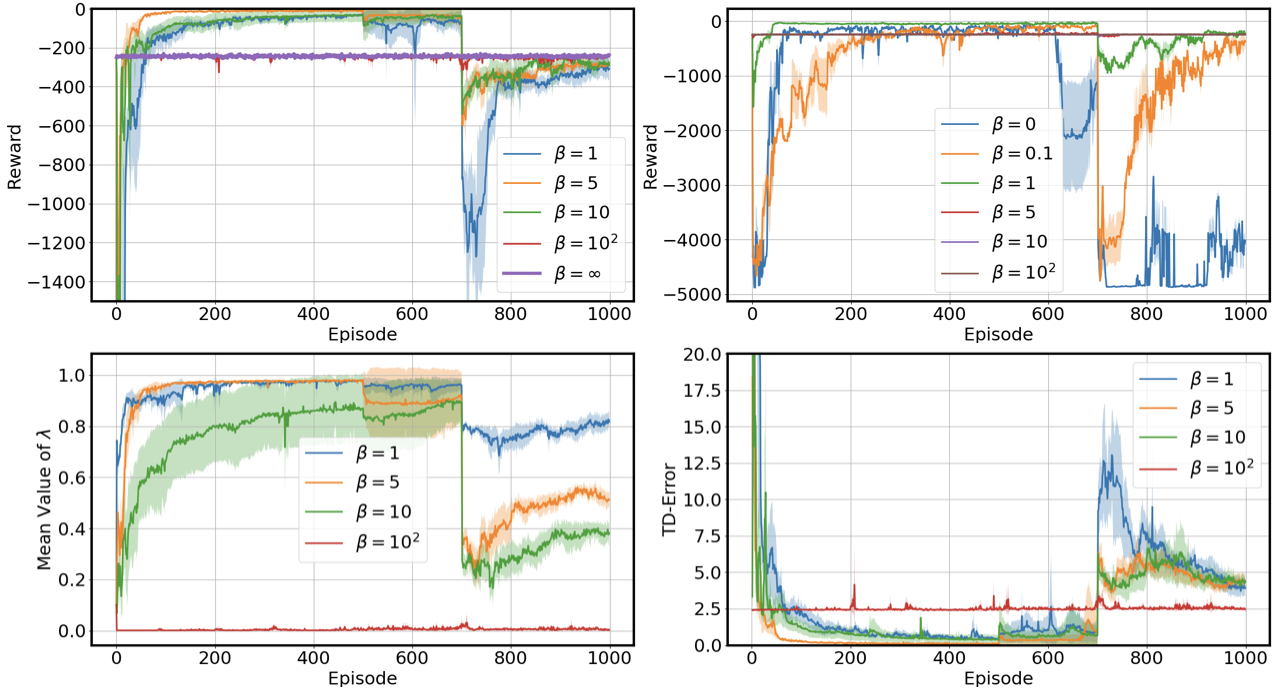}
\caption{Stability against distribution shifts for different choices of $\beta$ in the empirical heuristic budget $R_t^{\mathrm{emp}}$.}
\label{fig:nonstationary}
\end{figure}

\paragraph{Non-Stationary Environment}

In a subsequent experiment, we address scenarios where there is a distribution shift in the underlying MDP. We use the same matrices $A,B,Q$, and $R$ in~\eqref{eq:A_and_B} and~\eqref{eq:Q_and_R}.

For each $w_t$ in~\eqref{eq:canonical2}, we treat it as an independent Gaussian vector. Specifically, every entry $w_t (i)$ of $w_t$ is considered as an independent Gaussian random variable.  For the first $700$ episodes, each $w_t (i)$ is sampled from a normal distribution $\mathcal{N}(\mu,\sigma)$ where $\mu=0.5$ and $\sigma=0.05$. However, in the last $300$ episodes, we adjust $\mu$ to $-0.5$.

In the context of this nonstationary MDP, Figure~\ref{fig:nonstationary} illustrates the reward recovery after the occurrence of a distribution shift for varying choices of $\beta$ in the empirical heuristic budget $R_t^{\mathrm{emp}}$.
The two top figures show the average rewards. The top left figure uses the action space \(\mathcal{U}=[-100,100]\), while the top right uses \(\mathcal{U}=[-5,5]\). The bottom left figure presents the average behavior of \(\lambda(R_t^{\mathrm{emp}})\), and the bottom right presents the average absolute sampled TD error for \(\mathcal{U}=[-100,100]\). With \(\beta=10\) for \(\mathcal{U}=[-100,100]\) and \(\beta=1\) for \(\mathcal{U}=[-5,5]\), these tests show a favorable empirical balance between reward before the shift and recovery afterward. These observations concern the empirical heuristic and should not be interpreted as validating the corrected Grey Box theorem. As in the first set of experiments, we did not explicitly tune \(\beta\).

\newpage

\section{Useful Lemmas}

In this appendix, we present results that will be used when proving our main theorems.

\subsection{Perturbation Lemma}

We first prove the following perturbation lemma as a robustness guarantee, which holds for both the black-box (Section~\ref{sec:black_implementation}) and grey-box (Section~\ref{sec:gray_implementation}) settings.
\begin{lemma}[Perturbation Lemma]
\label{lemma:robustness}
Under Assumption~\ref{ass:Lip} and~\ref{ass:robustness},
the dynamic regret of
Algorithm~\eqref{alg:ppp} (denoted by \ouralg) can be bounded by
$\mathsf{DR}(\ouralg) \leq (\mathsf{ROB}-1)J^\star + L_{\mathrm{C}} C_s\sum_{t\in [\nt]} \mathbb{E}\left[\|u_t-\overline\pi_t(x_t)\|_{\mathcal U}^p\right]^{1/p}$
for constants $L_{\mathrm{C}}, C_s>0$.
\end{lemma}

\begin{proof}

Our proof consists of two parts. We first bound the Wasserstein distance between joint action-state distributions for the robust baseline and \ouralg. Next, we bound the dynamic regret.

\textbf{Step 1. Wasserstein Distance between Joint Action-State Distributions.}
Denote by $\pi$ and $\overline{\pi}$ the \ouralg policy and the robust baseline. Enlarge $s(0)$ to $\max\{s(0),1\}$ if necessary and adjust $C_s$ accordingly. Let $\rho_t$ and $\overline\rho_t$ be their state-action laws at time $t$. For $k\in\{-1,0,\ldots,t\}$, define the hybrid law $\rho_t^{(k)}$ by using $\pi$ through time $k$ and $\overline\pi$ afterward. Thus $\rho_t^{(-1)}=\overline\rho_t$ and $\rho_t^{(t)}=\rho_t$. The triangle inequality gives
\begin{equation}
\label{eq:norm_bound}
W_p(\rho_t,\overline\rho_t)
\leq\sum_{k=0}^t W_p\left(\rho_t^{(k)},\rho_t^{(k-1)}\right).
\end{equation}
The two hybrids in the $k$th term have the same state law at time $k$ and differ only in the action applied there.  Coupling them through the common state gives
\begin{equation}
\label{eq:hybrid_input}
W_p\left(\rho_k^{(k)},\rho_k^{(k-1)}\right)
\leq
\left(\mathbb E_{P,\pi}\left[\|u_k-\overline\pi_k(x_k)\|_{\mathcal U}^p\right]\right)^{1/p}
\leq\gamma.
\end{equation}
Propagating both coupled laws with the baseline from time $k$ to time $t$ and applying Definition~\ref{def:robust} yields
\begin{equation}
\label{eq:hybrid_propagation}
W_p\left(\rho_t^{(k)},\rho_t^{(k-1)}\right)
\leq s(t-k)\left(\mathbb E_{P,\pi}\left[\|u_k-\overline\pi_k(x_k)\|_{\mathcal U}^p\right]\right)^{1/p}.
\end{equation}
Combining~\eqref{eq:norm_bound} and~\eqref{eq:hybrid_propagation},
\begin{align}
\label{eq:action}
{W}_p\left(\rho_t,\overline{\rho}_t\right) \leq \sum_{\tau=0}^{t}s(\tau) \mathbb{E}_{P,\pi}\left[\|u_{t-\tau}-\overline\pi_{t-\tau}(x_{t-\tau})\|_{\mathcal U}^p\right]^{1/p}.
\end{align}


\textbf{Step 2. Dynamic Regret Analysis.}
Since the cost functions $(c_t:t\in [\nt])$  are Lipschitz continuous with a Lipschitz constant $L_{\mathrm{C}}$, using the Kantorovich-Rubinstein duality theorem~\cite{kantorovitch1958translocation}, since $ {W}_p(\mu,\nu)\leq  {W}_q(\mu,\nu)$ for all $1\leq p\leq q<\infty$, for all $t\in [\nt]$,
\begin{align}
\nonumber
&\mathbb{E}_{(x,u)\sim \rho_{t}}\left[c_t\left(x,u\right)\right]-\mathbb{E}_{(x,u)\sim \overline{\rho}_{t}}\left[c_t\left(x,u\right)\right] \\
\label{eq:kantorovich}
\leq & \sup_{\|f\|_{L}\leq L_{\mathrm{C}}}\mathbb{E}_{(x,u)\sim \rho_{t}}\left[f\left(x,u\right)\right]-\mathbb{E}_{(x,u)\sim \overline{\rho}_{t}}\left[f\left(x,u\right)\right]
\leq   L_{\mathrm{C}} W_p\left(\rho_{t},\overline{\rho}_{t}\right),
\end{align}
where $\|\cdot\|_L$ denotes the Lipschitz semi-norm and the supremum is over all Lipschitz 
 continuous functions $f$ with a Lipschitz constant $ L_{\mathrm{C}} $.
Therefore,
the difference between the expected cost of Algorithm~\ref{alg:ppp}, denoted by $\pi$, and the baseline policy $\overline{\pi}$ satisfies
\begin{align}
\nonumber
J(\pi) - J(\overline{\pi}) & =  \sum_{t\in [\nt]}\left( \mathbb{E}_{(x,u)\sim \rho_{t}}\left[c_t\left(x,u\right)\right]-\mathbb{E}_{(x,u)\sim \overline{\rho}_{t}}\left[c_t\left(x,u\right)\right]\right) \\
\nonumber
    & \leq L_{\mathrm{C}} \sum_{t\in [\nt]} \sum_{\tau=0}^{t}s(\tau) \mathbb{E}_{P,\pi}\left[\|u_{t-\tau}-\overline\pi_{t-\tau}(x_{t-\tau})\|_{\mathcal U}^p\right]^{1/p}\\
    \label{eq: final_bound_robustness}
    & \leq L_{\mathrm{C}} C_{s}\sum_{t\in [\nt]} \mathbb{E}_{P,\pi}\left[\|u_t-\overline\pi_t(x_t)\|_{\mathcal U}^p\right]^{1/p}
\end{align}
where we have used the  assumption of the $r$-locally robustness policy so that $ \sum_{t\in [\nt]}s(t)\leq C_{s}$ for some constant $C_{s} >0$. Moreover, since the robust baseline $\overline{\pi}$ has a ratio of expectations bound such that
$\frac{J(\overline{\pi})}{J^{\star}}\leq  \mathsf{ROB}.
$ Using Assumption~\ref{ass:Lip}, from~\eqref{eq: final_bound_robustness},
we obtain
\begin{align*}
 \mathsf{DR}(\ouralg)\coloneqq J(\pi) - J^{\star} \leq 
 (\mathsf{ROB}-1)J^\star + L_{\mathrm{C}} C_{s}\sum_{t\in [\nt]}\mathbb{E}_{P,\pi}\left[\|u_t-\overline\pi_t(x_t)\|_{\mathcal U}^p\right]^{1/p}.
\end{align*}
The projection identity also gives
\[
\|u_t-\overline\pi_t(x_t)\|_{\mathcal U}
=\min\{R_t,\eta_t\}\leq\eta_t\leq\gamma.
\]
\end{proof}

\subsection{Projection Lemma}
\label{app:projection_lemma}

The following lemma implies a useful consistency bound. It is worth noting that the lemma also holds if \ouralg adopts an alternative approach instead of projecting the actions as shown in~\eqref{eq:projection_definition}:
\begin{align*}
u_t \in \argmin_{v\in\mathsf{U}} \widetilde{Q}_t\left(x_t,v\right)\
\text{subject to } \left\|\overline{u}_t-v\right\|\leq R_t
\end{align*}
Implementing the projection rule in \ouralg can significantly reduce computational complexity, particularly when dealing with non-convex Q-advice.

\begin{lemma}[Projection Lemma]
\label{lemma:consistency}
Under Assumption~\ref{ass:Q-Lip}, the actions and states $(x_t,u_t)$ at $t\in [\nt]$ and $t\in [\nt]$ generated by \ouralg (Algorithm~\ref{alg:ppp}) satisfy
\begin{align}
\label{eq:consistency_perstep}
Q_{t}^{\star}\left(x_t,u_t\right)-\inf_{v\in\mathcal{U}} Q_{t}^{\star}\left(x_t,v\right)\leq L_Q \left([\eta_t\left(x_t\right)-R_t]^{+}\right)+  \mu_t\left(x_t,u_t\right)
\end{align}
where $\eta_t\left(x\right)\coloneqq \left\|\widetilde{\pi}_t\left(x\right)- \overline{\pi}_t\left(x\right)\right\|_{\mathcal{U}}$, and $Q^{\star}$ denotes the optimal {Q}-value functions satisfying the Bellman optimality equations in~\eqref{eq:bellman_optimal}.
\end{lemma}

\begin{proof}
Choose $\widetilde u_t\in\argmin_{v\in\mathcal U}\widetilde Q_t(x_t,v)$ and define
\[
\mu_t(x,u)
\coloneqq
\left(\inf_{v\in\mathcal U}\widetilde Q_t(x,v)
-\inf_{v\in\mathcal U}Q_t^\star(x,v)\right)
-\left(\widetilde Q_t(x,u)-Q_t^\star(x,u)\right).
\]
If $R_t<\eta_t$, convexity makes the radial point at distance $R_t$ from $\overline\pi_t(x_t)$ feasible, and the reverse triangle inequality shows that it minimizes the distance to $\widetilde u_t$. The tie-breaking rule in~\eqref{eq:projection_definition} selects this point. If $R_t\geq\eta_t$, the advised action is feasible. Therefore
\[
\|u_t-\widetilde u_t\|_{\mathcal U}
=[\eta_t-R_t]^+.
\]
Lipschitz continuity of $\widetilde Q_t$ now gives
\begin{align*}
Q_t^\star(x_t,u_t)-\inf_{v\in\mathcal U}Q_t^\star(x_t,v)
&=\widetilde Q_t(x_t,u_t)-\inf_{v\in\mathcal U}\widetilde Q_t(x_t,v)+\mu_t(x_t,u_t)\\
&\leq\widetilde Q_t(x_t,\widetilde u_t)-\inf_{v\in\mathcal U}\widetilde Q_t(x_t,v)
+L_Q[\eta_t-R_t]^++\mu_t(x_t,u_t)\\
&=L_Q[\eta_t-R_t]^++\mu_t(x_t,u_t),
\end{align*}
which proves the claim.
\end{proof}

\subsection{Analysis of Approximate TD-Error}

The following result that rewrites the approximate TD-error (c.f.~\eqref{eq:td_error_approximate}) is useful.

\begin{lemma}
\label{lemma:td_error}
    Consider the approximate TD-error in~\eqref{eq:td_error_approximate} and assume that the displayed variables below are integrable. Suppose
    \begin{align*}
        \delta_t\left(u_{t-1},x_{t-1},x_{t}\right) \coloneqq & c_{t-1}\left(x_{t-1},u_{t-1}\right) + \inf_{v\in\mathcal{U}}\widetilde{Q}_{t}\left(x_{t},v\right) - \widetilde{Q}_{t-1}\left(x_{t-1},u_{t-1}\right).
    \end{align*}
    It follows that for any $1\leq t\leq\nt-1$,
    \begin{align*}
\mathbb{E}_{P,\pi}\left[\delta_t\left(u_{t-1},x_{t-1},x_{t}\right)\right] = \mathbb{E}_{P,\pi}\left[\zeta_{t}^{V}(x_t) - \zeta_{t-1}^{Q}(x_{t-1},u_{t-1})\right],
\end{align*}
where $\zeta_{t}^{Q}$ and $\zeta_{t}^{V}$ are defined as
\begin{align*}
\zeta_{t}^{Q}\left(x_t,u_t\right) &\coloneqq \widetilde{Q}_{t}\left(x_t,u_t\right)-Q_{t}^{\star}\left(x_t,u_t\right),\\
\zeta_{t}^{V}\left(x_t\right) &\coloneqq \inf_{v\in\mathcal{U}}\widetilde{Q}_{t}\left(x_t,v\right)-\inf_{v\in\mathcal{U}}Q_{t}^{\star}\left(x_t,v\right).
\end{align*}
\end{lemma}

\begin{proof}
Let $\widetilde V_t(x)=\inf_{v\in\mathcal U}\widetilde Q_t(x,v)$ and $V_t^\star(x)=\inf_{v\in\mathcal U}Q_t^\star(x,v)$. For $1\leq t\leq\nt-1$, let $\mathcal F_{t-1}$ contain the history, advice, and policy randomness through the selected pair $(x_{t-1},u_{t-1})$, immediately before $x_t$ is sampled, and define
\begin{equation}
\label{eq:bellman_innovation}
\xi_t\coloneqq V_t^\star(x_t)
-(\mathbb P_{t-1}V_t^\star)(x_{t-1},u_{t-1}).
\end{equation}
The Bellman equation gives
\[
Q_{t-1}^\star(x_{t-1},u_{t-1})
=c_{t-1}(x_{t-1},u_{t-1})
+(\mathbb P_{t-1}V_t^\star)(x_{t-1},u_{t-1}).
\]
Adding and subtracting $V_t^\star(x_t)$ and $Q_{t-1}^\star(x_{t-1},u_{t-1})$ in the definition of $\delta_t$ gives the pathwise identity
\begin{equation}
\label{eq:pathwise_td}
\delta_t=\zeta_t^V-\zeta_{t-1}^Q+\xi_t.
\end{equation}
Assuming integrability, since $x_t\sim P_{t-1}(\cdot\mid x_{t-1},u_{t-1})$, $\mathbb E[\xi_t\mid\mathcal F_{t-1}]=0$. Taking expectations proves the stated identity. If $\mu_t=\zeta_t^V-\zeta_t^Q$ and $M_t=\sum_{s=1}^t\xi_s$, then
\begin{align}
\label{eq:correct_telescope}
\sum_{s=1}^t(\mu_s-\delta_s)
&=\zeta_0^Q-\zeta_t^Q-M_t,\\
\label{eq:causal_td}
\sum_{s=1}^t\delta_s
&=\sum_{j=0}^{t-1}\mu_j+\zeta_t^V-\zeta_0^V+M_t.
\end{align}
The first equality follows by summing $\mu_t-\delta_t=\zeta_{t-1}^Q-\zeta_t^Q-\xi_t$. The second is an equivalent rearrangement.

For the corrected statistic in~\eqref{eq:corrected_td}, define the realized optimal advantage and the advice gap by
\[
a_t\coloneqq Q_t^\star(x_t,u_t)-V_t^\star(x_t),
\qquad
g_t\coloneqq\widetilde Q_t(x_t,u_t)-\widetilde V_t(x_t).
\]
Since $\mu_t=a_t-g_t$, Equation~\eqref{eq:pathwise_td} gives
\begin{equation}
\label{eq:augmented_td_identity}
\widehat\delta_t
=a_{t-1}+\zeta_t^V-\zeta_{t-1}^V+\xi_t.
\end{equation}
Consequently, with $A_t\coloneqq\sum_{j=0}^{t-1}a_j$ and $G_t$ defined in~\eqref{eq:grey_evidence},
\begin{equation}
\label{eq:augmented_td_sum}
G_t=A_t+\zeta_t^V-\zeta_0^V+M_t.
\end{equation}
Thus $G_t$ tracks the actual accumulated optimal advantage, up to the two endpoint errors and the martingale innovation.
\end{proof}

Next, we present our analysis of the black-box setting by proving Theorem~\ref{thm:black_dr} and Theorem~\ref{thm:black_impossibility}.

\section{Black-Box Consistency and Robustness Analysis}
\label{app:black_box}

\subsection{Proof of Theorem~\ref{thm:black_dr}}
\label{app:proof_black_dr}

Consider an MDP model with Assumption~\ref{ass:Lip},\ref{ass:robustness}, and~\ref{ass:Q-Lip}. We prove the theorem below.

\begin{theorem}
\label{app:black_box_dr}
Suppose the machine-learned policy $\widetilde{\pi}$ is $(\infty,\varepsilon)$-consistent.
The expected dynamic regret of \ouralg with the \textsc{Black-Box} Procedure is bounded by
\begin{align}
\nonumber
    \mathsf{DR}(\ouralg)\leq\min\Big\{\mathcal{O}(\varepsilon)+ \mathcal{O}((1-\lambda)\gamma\nt),
    \mathcal{O}\left(\left(\mathsf{ROB}+ \lambda\gamma-1\right)\nt\right) \Big\}
\end{align}
where $\varepsilon$ is defined in~\eqref{eq:def_epsilon}, $\gamma$ is the diameter of the action space $\mathcal{U}$, $\nt$ is the length of the time horizon, $\mathsf{ROB}$ is the ratio of expectations of the robust baseline $\overline{\pi}$, and $0\leq\lambda\leq 1$ is a hyper-parameter.
\end{theorem}

\textbf{Consistency Analysis.}
To show the first bound in Theorem~\ref{thm:black_dr} regarding the consistency result, we consider the following steps. For any $t\in [\nt]$, denote by $\left(x_t,u_t\right)$ the corresponding state and action generated by the projection pursuit policy \ouralg, denoted by $\pi$. The Bellman optimality equations~\eqref{eq:bellman_optimal} imply:
\begin{align}
\label{eq:bellman}
Q_{t}^{\star}\left(x_{t},u_{t}\right) = c_{t}\left(x_{t},u_{t}\right)+\mathbb{E}_{P}\left[\inf_{v\in\mathcal{U}} Q_{t+1}^{\star}\left(x_{h+1},v\right)\big | x_t,u_t \right].
\end{align}
Therefore the dynamic regret of the projection pursuit policy $\pi$ can be rewritten as
\begin{align}
\label{eq:proof_dynamic_1}
\mathsf{DR}(\ouralg)=& J(\pi)-J^{\star}
=  \left(\mathbb{E}_{P,\pi}\left[\sum_{t=0}^{\nt-1}c_t\left(x_t,u_t\right)\right] -\inf_{v\in\mathcal{U}}Q_{t,0}^{\star}\left(x_0,v\right)\right).
\end{align}
Combining~\eqref{eq:proof_dynamic_1} with~\eqref{eq:bellman}, we obtain the following cost-difference bound:
\begin{align*}
&\mathsf{DR}(\ouralg)\\
=& \sum_{t=0}^{\nt-1}\left(\mathbb{E}_{P,\pi}\left[Q_{t}^{\star}\left(x_t,u_t\right)\right]-\mathbb{E}_{P}\left[\inf_{v\in\mathcal{U}} Q_{t+1}^{\star}\left(x_{t+1},v\right)\right]\right)-\inf_{v\in\mathcal{U}}Q_{t,0}^{\star}\left(x_0,v\right)\\
=&Q_{t,0}^{\star}\left(x_0,u_0\right)-\inf_{v\in\mathcal{U}}Q_{t,0}^{\star}\left(x_0,v\right)+\sum_{t=1}^{\nt-1}\left(\mathbb{E}_{P,\pi}\left[Q_{t}^{\star}\left(x_t,u_t\right)\right]-\mathbb{E}_{P}\left[\inf_{v\in\mathcal{U}} Q_{t}^{\star}\left(x_t,v\right)\right]\right)\\
=&\sum_{t=0}^{\nt-1}\mathbb{E}_{P,\pi}\left[Q_{t}^{\star}\left(x_t,u_t\right)-\inf_{v\in\mathcal{U}} Q_{t}^{\star}\left(x_t,v\right)\right].
\end{align*}

Recall that for the \textsc{Black-Box} Procedure in Section~\ref{sec:black_implementation}, the robustness budget is set as $R_t=\lambda\eta_t$ for all $t\in [\nt]$.
Applying the bound in Lemma~\ref{lemma:consistency} gives the following consistency bound:
\begin{align*}
    \mathsf{DR}(\ouralg)=\mathcal{O}\left(\sum_{t=0}^{\nt-1}\left([\eta_t\left(x_t\right)-R_t]^{+}\right)+  \mathbb{E}_{P,\pi}\left[\mu_t\right]\right) \leq \mathcal{O}\left(\sum_{t=0}^{\nt-1}\mathbb{E}_{P,\pi}\left[\mu_t\right]\right) + \mathcal{O}\left((1-\lambda)\gamma\nt\right)
\end{align*}
since $\eta_t(x_t)\leq \gamma$ for all $t\in [\nt]$, and 
\begin{align*}
    \mu_t = \zeta_{t}^{V}-\zeta_{t}^{Q} = & \widetilde{Q}_{t}(x_t,\widetilde{u}_{t})-Q_{t}^{\star}(x_t,u_{t}^{\star}) - \left(\widetilde{Q}_{t}(x_t,u_t)-Q_{t}^{\star}(x_t,u_t)\right)\\
    \leq & \Big\|\inf_{v\in\mathcal{U}}\widetilde{Q}_{t}(x_t,v)-\inf_{v\in\mathcal{U}}Q_{t}^{\star}(x_t,v)\Big\|_{\infty} + \Big\|\widetilde{Q}_{t}(x_t,u_t)-Q_{t}^{\star}(x_t,u_t)\Big\|_{\infty}.
\end{align*}
Noting that the machine-learned policy $\widetilde{\pi}$ is $(\infty,\varepsilon)$-consistent, we obtain $\sum_{t=0}^{\nt-1}\mathbb{E}_{P,\pi}\left[\mu_t\right]\leq \varepsilon$. Hence, 
\begin{align}
    \label{eq:black_dr_1}
    \mathsf{DR}(\ouralg)=\mathcal{O}\left(\varepsilon\right) + \mathcal{O}\left((1-\lambda)\gamma\nt\right).
\end{align}

\textbf{Robustness Analysis.} Note that for any $t\in [\nt]$, $\eta_t(x_t)\leq \gamma$, where $\gamma$ is the diameter of the compact action space $\mathcal{U}$. Hence, noting the black-box setting of the robustness budget $R_t=\lambda\eta_t$ for all $t\in [\nt]$ and applying Lemma~\ref{lemma:robustness}, the sum of expected discrepancies, over all $t$ can be bounded by 
\begin{align}
\nonumber
    \mathsf{DR}(\ouralg)\leq &
\mathcal{O}\left((\mathsf{ROB}-1)\nt\right) + L_{\mathrm{C}} C_{s}\sum_{t\in [\nt]}\mathbb{E}_{P,\pi}\left[\left(\lambda\gamma\right)^p\right]^{1/p}\\
 \label{eq:black_dr_2}
 \leq &\mathcal{O}\left((\mathsf{ROB}+\lambda\gamma-1)\nt\right).
\end{align}
Combining~\eqref{eq:black_dr_1} and~\eqref{eq:black_dr_2}, we complete the proof.

\subsection{Proof of Theorem~\ref{thm:black_consistency_robustness}}

Let $\mathcal{MDP}$ be the set of all MDP models $\mathsf{MDP}(\mathcal{X},\mathcal{U},\nt,{P},c)$ satisfying Assumption~\ref{ass:Lip},\ref{ass:robustness}, and~\ref{ass:Q-Lip}. 
To prove Theorem~\ref{thm:black_consistency_robustness}, noting that by the definitions of consistency and robustness, we apply Theorem~\ref{thm:black_dr} to derive a bound on the worst-case ratio of expectations:
\begin{align*}
   \sup_{\mathcal{MDP}} \mathsf{RoE}(\varepsilon) \leq 1+ \sup_{\mathcal{MDP}}\frac{\mathsf{DR}(\ouralg)}{J^{\star}} \leq \min \left\{1+\mathcal{O}\left(\frac{\varepsilon}{\nt}\right)+\mathcal{O}((1-\lambda)\gamma),\mathsf{ROB}+\mathcal{O}(\lambda\gamma)\right\},
\end{align*}
which implies that \ouralg with the \textsc{Black-Box} Procedure is $(1+\mathcal{O}((1-\lambda)\gamma)$-consistent and $(\mathsf{ROB}+\mathcal{O}(\lambda\gamma))$-robust.

\subsection{Proof of Theorem~\ref{thm:black_impossibility}}

\begin{proof}
According to Lemma~\ref{lemma:consistency}, the expected dynamic regret of \ouralg satisfies 
\begin{align*}
\mathsf{DR}(\ouralg)
=&\sum_{t\in [\nt]}\mathbb{E}_{P,\pi}\left[Q_{t}^{\star}\left(x_t,u_t\right)-\inf_{v\in\mathcal{U}} Q_{t}^{\star}\left(x_t,v\right)\right],
\end{align*}
where $\pi$ denotes \ouralg.
For notational simplicity, we introduce the following notation:
\begin{align*}
\Delta {Q}_t^{\star}(P,\pi) \coloneqq & \mathbb{E}_{P,\pi}\left[Q_t^{\star}(x_t,u_t) - \inf_{v\in\mathcal{U}}Q_t^{\star}(x_t,v) \right], \\ 
\Delta \widetilde{Q}_t (P,\pi)\coloneqq & \mathbb{E}_{P,\pi}\left[\widetilde{Q}_t(x_t,u_t) - \inf_{v\in\mathcal{U}}\widetilde{Q}_t(x_t,v) \right].
\end{align*}
With the \textsc{Black-Box} Procedure, we set $R_t = \lambda \eta_t$ with some hyper-parameter $0\leq \lambda \leq 1$.
Therefore, there exists Lipschitz continuous {Q}-value predictions $(\widetilde{Q}_t:t\in [\nt])$ with a Lipschitz constant $L_Q$ such that 
\begin{align}
\label{eq:proof_lower_bdd_1}
\mathsf{DR}\left(\ouralg({\textsc{Black-Box}})\right)
\geq &\sum_{t\in [\nt]} \left(\Delta {Q}_t^{\star}(P,\pi) 
 - \Delta\widetilde{Q}_t (P,\pi) + (1-\lambda)L_Q\gamma\right).
\end{align}
First, we verify that Wasserstein robust policies exist since we can construct a transition probability $P$ such that the states in different times are independent.
Denote by $\mathsf{OPT}$ the expected optimal total cost.
We can construct cost functions $(c_t:t\in [\nt])$ that are Lipschitz continuous with a Lipschitz constant $L_c$ and Q-advice $(\widetilde{Q}_t:t\in [\nt])$ satisfying
\begin{align*}
\frac{\sum_{t\in [\nt]} \left(\Delta {Q}_t^{\star}(P,\pi) - \Delta\widetilde{Q}_t (P,\pi) \right)}{\mathsf{OPT}} \geq \Omega\left(\mathsf{ROB}+
 \frac{\lambda\gamma  L_c}{\mathsf{OPT}}\nt \right).
\end{align*}
Note that the corresponding Q-value predictions satisfy Assumption~\ref{ass:Q-Lip}.
Let $\mathcal{MDP}$ be the set of all MDP models $\mathsf{MDP}(\mathcal{X},\mathcal{U},\nt,{P},c)$ satisfying Assumption~\ref{ass:Lip},\ref{ass:robustness}, and~\ref{ass:Q-Lip}.
Combining above with~\eqref{eq:proof_lower_bdd_1}, and noting that in Assumption~\ref{ass:Lip}, $c_t(x,u)>0$ for all $t\in [\nt]$, $x\in\mathcal{X}$, and $u\in\mathcal{U}$, for any $\varepsilon\geq 0$, the ratio of expectations can be bounded by
\begin{align*}
    \mathsf{RoE}(\ouralg) = 1+ \sup_{\mathcal{MDP}}\frac{\mathsf{DR}(\ouralg)}{\mathsf{OPT}} = 1+ \Omega\Big( (1-\lambda)L_Q\gamma  + \min\{\varepsilon,\lambda \gamma L_c +\mathsf{ROB}\}\Big),
\end{align*}
which implies that $\ouralg$ cannot be both $\left(1+o(\lambda\gamma)\right)$-consistent and $(\mathsf{ROB}+o((1-\lambda)\gamma))$-robust for any $0\leq\lambda\leq 1$.
\end{proof}

\section{Grey-Box Consistency and Robustness Analysis} 
\label{app:grey_box}

In the following, we present a dynamic regret bound for the grey-box setting (Section~\ref{sec:gray_implementation}) that is analogous to the one presented in Theorem~\ref{thm:black_dr} for the black-box scenario.

First, in addition to Definition~\ref{def:tradeoff}, we further recall the following quantities used in Lemma~\ref{lemma:td_error} for notational convenience:
\begin{align}
\label{eq:zeta_Q}
\zeta_{t}^{Q}\left(x_t,u_t\right) &\coloneqq \widetilde{Q}_{t}\left(x_t,u_t\right)-Q_{t}^{\star}\left(x_t,u_t\right),\\
\label{eq:zeta_V}
\zeta_{t}^{V}\left(x_t\right) &\coloneqq \inf_{v\in\mathcal{U}}\widetilde{Q}_{t}\left(x_t,v\right)-\inf_{v\in\mathcal{U}}Q_{t}^{\star}\left(x_t,v\right),
\end{align}
where by definition, $\zeta_{t}^{Q}$ and $\zeta_{t}^{V}$ depend on the random trajectory $((x_t,u_t):t\in [\nt])$.
Denote $\mu_t\coloneqq\zeta_{t}^{V}-\zeta_{t}^{Q}$.

Note that  when the environment is stationary, under some model assumptions and with a Reproducing kernel Hilbert space (RKHS) being the function class, the optimism lemma (Lemma 5.2) in~\cite{yang2020function} shows that with probability at least $1-(2\nt^2\nh^2)^{-1}$, the generated {Q}-value functions satisfy
$\sum_{(h,t)\in [\nh]\times  [\nt]}\mathbb{E}_{P,\pi}\left[\delta_{h,t}+ \mu_{h,t}\right] = \widetilde{O}(\nh {\Gamma_K(\nt,\lambda)}\sqrt{\nt})$ where $\nh$ is the number of episodes and $\widetilde{O}(\cdot)$ omits logarithmic terms and ${\nt\Gamma_K(\nt,\lambda)}$ is the  maximal information gain~\cite{srinivas2010gaussian} that characterizes the intrinsic complexity of the function class.

\begin{theorem}[Grey-Box: Dynamic Regret]
\label{thm:grey_dynamic_regret}
Consider any MDP model satisfying Assumption~\ref{ass:Lip},\ref{ass:robustness}, and~\ref{ass:Q-Lip}. The expected dynamic regret of \ouralg (Algorithm~\ref{alg:ppp}) with the \textsc{Grey-Box} Procedure satisfies the following bounds:
\begin{align}
\label{eq:global_consistency}
\mathsf{DR}(\ouralg)
&\leq
\mathbb E_{P,\pi}\sum_{t\in[\nt]}
\left(\mu_t+L_Q[\eta_t-R_t]^+\right),\\
\label{eq:global_robustness}
\mathsf{DR}(\ouralg)
&\leq
(\mathsf{ROB}-1)J^\star
+L_{\mathrm C}C_s\sum_{t\in[\nt]}
\left(\mathbb E_{P,\pi}
\left[\|u_t-\overline\pi_t(x_t)\|_{\mathcal U}^p\right]\right)^{1/p}.
\end{align}
Consequently, the minimum of the two global right-hand sides is also an upper bound.

For the budget in~\eqref{eq:budget}, the two bounds imply
\begin{align}
\label{eq:global_ratio}
\mathsf{RoE}(\ouralg)
\leq\min\Bigg\{&
1+\frac{1}{J^\star}\mathbb E_{P,\pi}\sum_{t\in[\nt]}
\left(\mu_t+\min\{L_Q\eta_t,\beta Z_t\}\right),\notag\\
&\mathsf{ROB}+\frac{L_{\mathrm C}C_s}{J^\star}
\sum_{t\in[\nt]}\left(\mathbb E_{P,\pi}
\left[\|u_t-\overline\pi_t(x_t)\|_{\mathcal U}^p\right]\right)^{1/p}
\Bigg\}.
\end{align}
\end{theorem}

\subsection{Proof of Theorem~\ref{thm:grey_dynamic_regret}}
\label{appendix:thm_grey-proof}

The Bellman equations give the performance-difference identity
\[
\mathsf{DR}(\ouralg)
=\mathbb E_{P,\pi}\sum_{t\in[\nt]}
\left(Q_t^\star(x_t,u_t)-\inf_{v\in\mathcal U}Q_t^\star(x_t,v)\right).
\]
Applying the Projection Lemma to each summand proves~\eqref{eq:global_consistency}. The Perturbation Lemma proves~\eqref{eq:global_robustness}. Both inequalities hold for the full horizon, so their minimum is also an upper bound.

The budget identity gives
\[
L_Q[\eta_t-R_t]^+
=\min\{L_Q\eta_t,\beta Z_t\}.
\]
Substituting this identity into~\eqref{eq:global_consistency} and dividing both global bounds by $J^\star$ proves~\eqref{eq:global_ratio}. The consequences under exact and arbitrary advice are proved next.

\subsection{Proof of Theorem~\ref{thm:grey_robustness_consistency}}
\label{appendix:thm_grey_consistency}

\paragraph{Conditions and finite horizon bounds.}
Consider a class of horizon-$\nt$ MDPs satisfying Assumptions~\ref{ass:Lip}, \ref{ass:robustness}, and~\ref{ass:Q-Lip}, with constants uniform over the class. Suppose there are constants $0<c_0\leq c_1<\infty$ such that
\[
c_0\leq c_t(x,u)\leq c_1
\]
for every $t$, $x$, and $u$, so $J^\star\geq c_0\nt$. Let $\xi_t$ and $M_t$ be the Bellman innovation and martingale in Lemma~\ref{lemma:td_error}. Assume that a deterministic nondecreasing boundary $(b_{\nt,t}:0\leq t<\nt)$, with $b_{\nt,0}=0$, satisfies
\[
\mathbb P\left(\left|M_t\right|>b_{\nt,t}
\text{ for some }1\leq t<\nt\right)\leq\alpha_{\nt}
\]
uniformly over the admissible adaptive policies. Write $b_{\nt}\coloneqq b_{\nt,\nt-1}$. Finally, define
\[
J_t^{\overline\pi}(x)
\coloneqq
\mathbb E_{P,\overline\pi}\left[\sum_{s=t}^{\nt-1}c_s(x_s,u_s)\mid x_t=x\right]
\]
with $J_{\nt}^{\overline\pi}(x)=V_{\nt}^\star(x)=0$, and suppose the baseline guarantee is restartable in the sense that, for a class-uniform $\rho_{\nt}=o(\nt)$,
\begin{equation}
\label{eq:restartable_baseline}
J_t^{\overline\pi}(x)
\leq \mathsf{ROB}\,V_t^\star(x)+\rho_{\nt}
\end{equation}
for every $0\leq t\leq\nt$ and every reachable $x\in\mathcal X$.

For each horizon, set $\beta=\beta_{\nt}$ in the \textsc{Grey-Box} Procedure. Let $\overline a_{\nt}\coloneqq L_Q\gamma+2q_{\nt}$ and
\[
K_{\nt}\coloneqq
\frac{L_Q\gamma}{\beta_{\nt}}+2b_{\nt}+L_Q\gamma+4q_{\nt}.
\]
The consistency factor under exact Q-value advice and the robustness factor over the admissible advice class are respectively bounded by
\begin{equation}
\label{eq:grey_finite_factors}
1+\frac{\alpha_{\nt}L_Q\gamma}{c_0}
\qquad\text{and}\qquad
\mathsf{ROB}+
\frac{\mathsf{ROB}\left(K_{\nt}+\alpha_{\nt}\nt\overline a_{\nt}\right)+\rho_{\nt}}{c_0\nt}.
\end{equation}
Consequently, Theorem~\ref{thm:grey_robustness_consistency} follows if $\alpha_{\nt}\to0$, $\alpha_{\nt}\overline a_{\nt}=o(1)$, $b_{\nt}=o(\nt)$, and $1/\beta_{\nt}=o(\nt)$, together with $q_{\nt}=o(\nt)$ and $\rho_{\nt}=o(\nt)$ stated above.

The confidence condition is a concrete property of $P_t$ and $V_t^\star$. For example, if the innovations are conditionally $\sigma_{\nt}^2$-sub-Gaussian uniformly over the class, one may take
\[
b_{\nt,t}=\sigma_{\nt}\sqrt{2t\log(2\nt/\alpha_{\nt})}.
\]
Uniformly Lipschitz $V_t^\star$ and transition distributions with uniformly bounded support diameter give such a condition by Hoeffding's lemma. If $\sigma_{\nt}=\mathcal O(1)$, $\alpha_{\nt}=\nt^{-2}$, and $\beta>0$ is fixed, then the two factors in~\eqref{eq:grey_finite_factors} are
\[
1+\mathcal O(\nt^{-2})
\qquad\text{and}\qquad
\mathsf{ROB}+\mathcal O\left(
\sqrt{\frac{\log\nt}{\nt}}+
\frac{q_{\nt}+\rho_{\nt}}{\nt}+
\frac{1}{\beta\nt}
\right).
\]
Condition~\eqref{eq:restartable_baseline} makes explicit that the baseline guarantee must hold after a data-dependent switching time. It holds whenever the same baseline ratio applies uniformly to every restarted suffix MDP, up to the sublinear term $\rho_{\nt}$.

\paragraph{Proof.}

Let
\[
\mathcal E_{\nt}
\coloneqq
\left\{|M_t|\leq b_{\nt,t}
\text{ for every }1\leq t<\nt\right\}.
\]
By assumption, $\mathbb P(\mathcal E_{\nt}^c)\leq\alpha_{\nt}$. We first record two pathwise bounds. The projection identity and the Lipschitz continuity of $\widetilde Q_t$ give
\begin{equation}
\label{eq:advice_gap_bound}
0\leq g_t
\leq L_Q\|u_t-\widetilde u_t\|_{\mathcal U}
=L_Q[\eta_t-R_t]^+
\leq L_Q\gamma.
\end{equation}
Moreover, Assumption~\ref{ass:Q-Lip} gives
\begin{equation}
\label{eq:advantage_envelope}
0\leq a_t
\leq g_t+2q_{\nt}
\leq \overline a_{\nt}.
\end{equation}

Let
\[
\kappa_{\nt}\coloneqq\frac{L_Q\gamma}{\beta_{\nt}},
\qquad
\tau\coloneqq\inf\{t\in\{1,\ldots,\nt-1\}:Z_t\geq\kappa_{\nt}\},
\]
with $\tau=\nt$ if the set is empty. Since $Z_t$ is nondecreasing and $\eta_t\leq\gamma$, Equation~\eqref{eq:budget} gives $R_t=0$ for every $t\geq\tau$. Thus \ouralg follows the baseline from $\tau$ onward.

On $\mathcal E_{\nt}$, the accumulated advantage before the switch satisfies
\begin{equation}
\label{eq:prefix_advantage_bound}
A_{\tau}\leq K_{\nt}.
\end{equation}
To see this, if $2\leq\tau<\nt$, then $Z_{\tau-1}<\kappa_{\nt}$ and hence $G_{\tau-1}<b_{\nt,\tau-1}+\kappa_{\nt}$. Equation~\eqref{eq:augmented_td_sum} gives
\[
A_{\tau-1}
\leq \kappa_{\nt}+2b_{\nt}+2q_{\nt}.
\]
The crossing action satisfies $a_{\tau-1}\leq L_Q\gamma+2q_{\nt}$ by~\eqref{eq:advantage_envelope}, which proves~\eqref{eq:prefix_advantage_bound}. The case $\tau=1$ follows directly from~\eqref{eq:advantage_envelope}. If $\tau=\nt$, the same argument at time $\nt-1$, followed by the bound on the final action $a_{\nt-1}$, proves~\eqref{eq:prefix_advantage_bound}. On $\mathcal E_{\nt}^c$, Equation~\eqref{eq:advantage_envelope} gives $A_\tau\leq\nt\overline a_{\nt}$. Therefore
\begin{equation}
\label{eq:expected_prefix_advantage}
\mathbb E[A_\tau]
\leq K_{\nt}+\alpha_{\nt}\nt\overline a_{\nt}.
\end{equation}

The time $\tau$ is a bounded stopping time for the trajectory filtration. With $V_{\nt}^\star=0$, optional stopping in the Bellman performance-difference identity gives
\begin{equation}
\label{eq:stopped_performance_difference}
\mathbb E\left[\sum_{t=0}^{\tau-1}c_t(x_t,u_t)+V_\tau^\star(x_\tau)\right]
=J^\star+\mathbb E[A_\tau].
\end{equation}
Applying the restartable baseline bound~\eqref{eq:restartable_baseline} conditionally at $\tau$ and using nonnegative costs,
\begin{align*}
J(\ouralg)
&\leq
\mathbb E\left[\sum_{t=0}^{\tau-1}c_t(x_t,u_t)
+\mathsf{ROB}V_\tau^\star(x_\tau)\right]+\rho_{\nt}\\
&\leq
\mathsf{ROB}\left(J^\star+\mathbb E[A_\tau]\right)+\rho_{\nt}.
\end{align*}
Combining this inequality with~\eqref{eq:expected_prefix_advantage} and $J^\star\geq c_0\nt$ proves the robustness factor in~\eqref{eq:grey_finite_factors}.

Finally, suppose the Q-value advice is exact. On $\mathcal E_{\nt}$, induction gives $G_t=M_t$, $Z_t=0$, $R_t=\eta_t$, and $a_t=g_t=0$ for every $t$. Indeed, this is true at $t=0$. If it is true before time $t$, Equation~\eqref{eq:augmented_td_sum} gives $G_t=M_t\leq b_{\nt,t}$, so the next action again follows the optimal Q-value advice. Hence $A_{\nt}=0$ on $\mathcal E_{\nt}$. On its complement, exact advice and~\eqref{eq:advantage_envelope} give $A_{\nt}\leq\nt L_Q\gamma$. The performance-difference identity therefore yields
\[
\mathsf{DR}(\ouralg)
=\mathbb E[A_{\nt}]
\leq\alpha_{\nt}\nt L_Q\gamma.
\]
Dividing by $J^\star\geq c_0\nt$ proves the consistency factor. If transitions are deterministic, then $M_t=0$ and the same induction gives exact consistency without a failure event.

\section{Proof of Theorem~\ref{thm:W-Robustness-and-ROB}}\label{appendix:MPC-proof}
To show Theorem \ref{thm:W-Robustness-and-ROB}, we first show a technical lemma with respect to $\MPC_\nt$, which plans until the end of the episode from the first time step.

\begin{lemma}\label{lemma:infinite-horizon-MPC}
Suppose Assumptions \ref{assump:bounded-costs-and-dynamics} and \ref{assump:uniform-stability} hold. For each step $t \in [\nt]$, the control policy of $\MPC_{\nt}$ can be rewritten as $u_t = \overline{K}_t x_t$, for some matrices $\left(\overline{K}_t:{t \in [\nt]}\right)$ satisfy that $\norm{\overline{K}_t} \leq C$ for all $t \in [\nt]$, and
\[\norm{(A_{t'-1} + B_{t'-1} \overline{K}_{t'-1}) \cdots (A_t + B_t \overline{K}_t)} \leq C \decayfactor^{t' - t}, \forall t, t' \in [\nt], \ t' \geq t,\]
where $\lambda, C$ are as defined in Theorem \ref{thm:W-Robustness-and-ROB}.
\end{lemma}
\begin{proof}
To simplify the notation, we define 
\begin{align}\label{lemma:infinite-horizon-MPC:e0}
    \Gamma_{t, t'} = 
    \begin{cases}
        \DIAG\left(Q_t, R_t, \ldots, R_{t'-1}, Q_{t'}\right) & \text{ if } t' = \nt - 1\\
        \DIAG\left(Q_t, R_t, \ldots, R_{t'-1}, P_{t'}\right) & \text{ otherwise}
    \end{cases}.
\end{align}

By the KKT conditions, we see that for any $t \in [\nt]$, the predictive optimal solution $\psi_{t, \nt-1}(x_t)$ is given by
\begin{align}\label{lemma:infinite-horizon-MPC:e1}
    \left(\begin{array}{c}
         x_{t|t}\\
         u_{t|t}\\
         \vdots\\
         x_{\nt-1|t}\\
         \hline
         \eta_{t|t}\\
         \vdots\\
         \eta_{\nt-1|t}
    \end{array}\right) = 
    \left(\begin{array}{c|c}
        \Gamma_{t, \nt-1} & (\Xi_{t, \nt-1})^\top \\
        \hline
        \Xi_{t, \nt-1} & 
    \end{array}\right)^{-1} \left(\begin{array}{c}
         0\\
         \vdots\\
         0\\
         \hline
         x_t\\
         0\\
         \vdots\\
         0
    \end{array}\right).
\end{align}
Therefore, $u_{t|t}$ is a linear function of $x_t$, and this relationship defines $\overline{K}_t$. Lemma G.2 in \cite{lin2022bounded} implies that $\norm{\overline{K}_t} \leq C$. Note that the block matrix is invertible since $(Q_t:t\in [\nt])$ and $(R_t:t\in [\nt])$ are positive definite.

To simplify the notation, we define the state transition matrix
\[\Phi_{t, t'} \coloneqq (A_{t'-1} + B_{t'-1} \overline{K}_{t'-1}) \cdots (A_t + B_t \overline{K}_t).\]

Consider an arbitrary state $x_t$. Note that $\left( x_{t'\mid t}:{t \leq t' < \nt}\right)$ is the optimal trajectory when there is no disturbance after step $t$. By the principle of optimality, we see that $\left(x_{t'\mid t}:t \leq t'< \nt\right)$ is identical with the actual trajectory of $\MPC_\nt$ after step $t$. In other words, for arbitrary $x_t$, the multi-step transition matrix $\Phi_{t, t'}$ satisfies
\[x_{t'\mid t} = \Phi_{t, t'} x_t.\]
Lemma G.2 in \cite{lin2022bounded} implies that $\norm{\Phi_{t, t'}} \leq C \decayfactor^{t'-t}$.
\end{proof}

Lemma \ref{lemma:infinite-horizon-MPC} shows that $\MPC_{\nt}$ has the same effect as a time-varying linear feedback controller that is exponentially stable. We generalize this property to $\MPC_k$ with a smaller prediction horizon (Lemma \ref{lemma:finite-horizon-MPC}) by showing that $\MPC_k$ behaves similar to $\MPC_{\nt}$ when $k$ is sufficiently large.

\begin{lemma}\label{lemma:finite-horizon-MPC}
Suppose Assumptions \ref{assump:bounded-costs-and-dynamics} and \ref{assump:uniform-stability} hold. Let $(C, \decayfactor)$ be the same as Lemma \ref{lemma:infinite-horizon-MPC}. For each step $t \in [\nt]$, the control policy of $\MPC_{k}$ can be rewritten as $u_t = K_t^{k} x_t$, for some matrices $\{K_t^{k}\}_{t \in [\nt]}$ satisfy that 
\[\norm{K_t^{k}} \leq C, \text{ and }\norm{K_t^{k} - \overline{K}_t} \leq C^2 a \cdot \decayfactor^{2k}.\]
Further, for any $\widehat{\decayfactor} > \decayfactor$, when
$k \geq \min\{\nt, \frac{1}{2}\log\left(C^3 b a \decayfactor / (\widehat{\decayfactor} - \decayfactor)\right)/\log(1/\decayfactor)\},$
we have
\[\norm{(A_{t'-1} + B_{t'-1} K_{t'-1}^{k}) \cdots (A_t + B_t K_t^{k})} \leq C \widehat{\decayfactor}^{t' - t}, \text{ for any } t, t' \in [\nt], t' \geq t.\]
\end{lemma}
\begin{proof}
Let $\overline{t} \coloneqq \min\{t+k, \nt-1\}$. By the KKT conditions, we see that for any $t \in [\nt]$, the predictive optimal solution $\psi_{t, \overline{t}}(x_t; P_{\overline{t}})$ is given by
\begin{align}\label{lemma:finite-horizon-MPC:e1}
    \left(\begin{array}{c}
         x_{t|t}\\
         u_{t|t}\\
         \vdots\\
         x_{\overline{t}|t}\\
         \hline
         \eta_{t|t}\\
         \vdots\\
         \eta_{\overline{t}|t}
    \end{array}\right) = 
    \left(\begin{array}{c|c}
        \Gamma_{t, \overline{t}} & (\Xi_{t, \overline{t}})^\top \\
        \hline
        \Xi_{t, \overline{t}} & 
    \end{array}\right)^{-1} \left(\begin{array}{c}
         0\\
         \vdots\\
         0\\
         \hline
         x_t\\
         0\\
         \vdots\\
         0
    \end{array}\right).
\end{align}
Therefore, $u_{t|t}$ is a linear function of $x_t$, and this relationship defines $K_t^{k}$. By Lemma G.2 of \cite{lin2022bounded}, we see that $\norm{K_t^{k}} \leq C$.

When $\overline{t} < \nt-1$, construct an auxiliary disturbance sequence $\widehat{w}_{t:\nt-2\mid t}$ with $\widehat{w}_{\overline{t}\mid t} \coloneqq - A_t \psi_{t, \overline{t}}(x_t)$ and $\widehat{w}_{t'\mid t} = 0$ for all $t'\not=\overline{t}$. We see that
\[\psi_{t, \overline{t}}(x_t)[u_{t\mid t}] = \psi_{t, \nt-1}(x_t, \widehat{w}_{t:\nt-2\mid t}; Q_{\nt-1})[u_{t\mid t}].\]
Therefore, we see that
\begin{subequations}\label{lemma:finite-horizon-MPC:e2}
\begin{align}
    &\norm{\psi_{t, \overline{t}}(x_t)[u_{t\mid t}] - \psi_{t, \nt-1}(x_t)[u_{t\mid t}]}\nonumber\\
    ={}& \norm{\psi_{t, \nt-1}(x_t, \widehat{w}_{t:\nt-2\mid t}; Q_{\nt-1})[u_{t\mid t}] - \psi_{t, \nt-1}(x_t, 0_{\times (\nt - t - 1)}; Q_{\nt-1})[u_{t\mid t}]}\nonumber\\
    \leq{}& C \decayfactor^{k} \norm{\widehat{w}_{\overline{t}\mid t}} \label{lemma:finite-horizon-MPC:e2:s1}\\
    \leq{}& C^2 a \cdot \decayfactor^{2k}\norm{x_t}, \label{lemma:finite-horizon-MPC:e2:s2}
\end{align}
\end{subequations}
where we have applied the perturbation bounds in Lemma G.2 of \cite{lin2022bounded} in \eqref{lemma:finite-horizon-MPC:e2:s1} and \eqref{lemma:finite-horizon-MPC:e2:s2}. Since this inequality holds for any arbitrary $x_t$, we see that $\norm{K_t^{k} - \overline{K}_t} \leq C^2 a \cdot \decayfactor^{2k}$. To simplify the notation, we denote $\epsilon \coloneqq C^2 a \cdot \decayfactor^{2k}$.

We can derive the following bound in terms of the $\ell_2$ norm:
\begin{subequations}\label{lemma:finite-horizon-MPC:e3}
\begin{align}
    &\norm{(A_{t'-1} + B_{t'-1} K_{t'-1}^{k}) \cdots (A_t + B_t K_t^{k})}\nonumber\\
    \leq{}& \sum_{j = 0}^{t' - t} \binom{t' - t}{j} C^{j+1}\decayfactor^{t - t'} (b\epsilon)^j\label{lemma:finite-horizon-MPC:e3:s1}\\
    ={}& C \decayfactor^{t' - t}\left(1 + C b \epsilon\right)^{t' - t}\nonumber\\
    \leq{}& C \widehat{\decayfactor}^{t' - t}, \label{lemma:finite-horizon-MPC:e3:s2}
\end{align}
\end{subequations}
where we use the decomposition that for any $t'' \in \{t, \ldots, t'-1\}$,
\[A_{t''} + B_{t''} K_{t''}^{k} \leq (A_{t''} + B_{t''} \overline{K}_{t''}) + B_{t''}(K_{t''}^{k} - \overline{K}_{t''})\]
and $\norm{B_{t''}(K_{t''}^{k} - \overline{K}_{t''})} \leq b \epsilon$ in \eqref{lemma:finite-horizon-MPC:e3:s1}. We also use Lemma \ref{lemma:infinite-horizon-MPC} in \eqref{lemma:finite-horizon-MPC:e3:s1} and the assumption that
\[k \geq \frac{1}{2}\log\left(C^3 b a \decayfactor / (\widehat{\decayfactor} - \decayfactor)\right)/\log(1/\decayfactor)\]
in \eqref{lemma:finite-horizon-MPC:e3:s2}.
\end{proof}

To establish a dynamic regret bound that depends on the offline optimal cost, we first need to show a lower bound of $J^*$ that depends on the ``power'' of the unknown disturbances.

\begin{lemma}\label{lemma:opt-lower-bound}
The offline optimal cost is lower bounded by
\[J^* \geq \frac{\mu}{4(1 + a^2 + b^2)}\sum_{t=0}^{\nt-2} \norm{w_t}^2.\]
\end{lemma}
\begin{proof}
Note that the dynamics of the LTV system can be rewritten as
\[x_{t+1} - A_t x_t - B_t u_t = w_t.\]
Taking norms on both sides of the equality gives
\begin{subequations}\label{lemma:opt-lower-bound:e1}
\begin{align}
    \norm{w_t} ={}& \norm{x_{t+1} - A_t x_t - B_t u_t}\nonumber\\
    \leq{}& \norm{x_{t+1}} + \norm{A_t x_t} + \norm{B_t u_t}\label{lemma:opt-lower-bound:e1:s1}\\
    \leq{}& \norm{x_{t+1}} + a\norm{x_t} + b\norm{u_t}, \label{lemma:opt-lower-bound:e1:s2}
\end{align}
\end{subequations}
where we have used the triangle inequality in \eqref{lemma:opt-lower-bound:e1:s1} and the definition of the induced matrix $\ell_2$-norm in \eqref{lemma:opt-lower-bound:e1:s2}. Taking the squares of both sides and applying the Cauchy-Schwartz inequality together imply
\begin{align}\label{lemma:opt-lower-bound:e2}
    \norm{w_t}^2 \leq \left(\norm{x_{t+1}} + a\norm{x_t} + b\norm{u_t}\right)^2 \leq \frac{1 + a^2 + b^2}{\mu} \left(\mu \norm{x_{t+1}}^2 + \mu \norm{x_t}^2 + \mu \norm{u_t}^2\right).
\end{align}
By \eqref{lemma:opt-lower-bound:e2} and the assumptions on $Q_t$ and $R_t$, we obtain that
\begin{align*}
    \frac{\mu}{2(1 + a^2 + b^2)}\cdot \sum_{t=0}^{\nt-1}\norm{w_t}^2 \leq{}& \frac{1}{2}\sum_{t=0}^{\nt-2} \left(\mu \norm{x_{t+1}}^2 + \mu \norm{x_t}^2 + \mu \norm{u_t}^2\right)\\
    \leq{}& 2 \sum_{t=0}^{\nt-1} c_t(x_t, u_t).
\end{align*}
Since the above inequality holds for any arbitrary trajectory $\left((x_t, u_t):t\in [\nt]\right)$, we conclude that Lemma \ref{lemma:opt-lower-bound} holds.
\end{proof}

Since the Wasserstein robustness (see Definition~\ref{def:robust}) is for distributions on the state-action space, we also prove a technical lemma below that helps convert the contraction on deterministic state/action pairs to the contraction on distributions. Let ${W}_1(\mu,\nu)$ denote the Wasserstein $1$-distance between two distributions $\mu$ and $\nu$.

\begin{lemma}\label{lemma:Wasserstain-preserves-contraction}
Suppose $\varphi: \mathcal{Y} \to \mathcal{W}$ is a deterministic function that satisfies $\norm{\varphi(v) - \varphi(v')}_{\mathcal{W}} \leq \kappa \norm{v - v'}_{\mathcal{Y}}$ for any $v, v' \in \mathcal{Y}$. Then, for any pair of distributions  $\rho$ and $\rho'$ on $\mathcal{Y}$, we have $W_1(\varphi(\rho), \varphi(\rho')) \leq \kappa W_1(\rho, \rho')$.
\end{lemma}
\begin{proof}
Recall that $W_1(\rho, \rho') \coloneqq \inf_{J} \int \norm{v - v'}_{\mathcal{Y}} d J(v, v')$, where $J$ is a joint distribution on $\mathcal{Y}\times \mathcal{Y}$ with marginals $\rho$ and $\rho'$. We define a mapping $\Phi: \mathcal{Y} \times \mathcal{Y} \to \mathcal{Z} \times \mathcal{Z}$ as $\Phi(v, v') \coloneqq (\varphi(v), \varphi(v'))$. We see that $\Phi J$ gives a joint distribution on $\mathcal{W}\times \mathcal{W}$ with marginals $\varphi(\rho)$ and $\varphi(\rho')$, and it satisfies
\[\int \norm{u - u}_{\mathcal{W}} \mathrm{d} (\Phi J)(u, u') = \int \norm{\varphi(v) - \varphi(v')}_\mathcal{Y} \mathrm{d} J(v, v') \leq \varepsilon \int \norm{v - v'}_{\mathcal{Y}} \mathrm{d} J(v, v').\]
Note that the above inequality holds for any $J$ with marginals $\rho$ and $\rho'$. Thus Lemma \ref{lemma:Wasserstain-preserves-contraction} holds.
\end{proof}

Now we are ready to show Theorem \ref{thm:W-Robustness-and-ROB}.

For a state $x$ at time step $t\in [\nt]$, let $x_{t:t'}(x)$ and $u_{t:t'}(x)$ denote the corresponding state and action of MPC at time step $t'$. By Lemma \ref{lemma:finite-horizon-MPC}, we see that for any state-action pairs $(x, u)$ and $(x', u')$ at step $t_1$, we have
\begin{subequations}\label{thm:MPC-baseline-complete:e1}
\begin{align}
    &\norm{(x_{t_1+1:t_2}, u_{t_1+1:t_2})(A_{t_1} x + B_{t_1} u + w_{t_1}) - (x_{t_1+1:t_2}, u_{t_1+1:t_2})(A_{t_1} x' + B_{t_1} u' + w_{t_1})}\nonumber\\
    \leq{}& (1 + C)\norm{x_{t_1+1:t_2}(A_{t_1} x + B_{t_1} u + w_{t_1}) - x_{t_1+1:t_2}(A_{t_1} x' + B_{t_1} u' + w_{t_1})}\label{thm:MPC-baseline-complete:e1:s1}\\
    \leq{}& (1 + C)C \widehat{\decayfactor}^{t_2 - t_1 - 1} \norm{A_{t_1}(x - x') + B_{t_1}(u - u')} \label{thm:MPC-baseline-complete:e1:s2}\\
    \leq{}& (1 + C)C(a + b) \widehat{\decayfactor}^{t_2 - t_1 - 1} \norm{(x, u) - (x', u')}, \label{thm:MPC-baseline-complete:e1:s3}
\end{align}
\end{subequations}
where we have used Lemma \ref{lemma:finite-horizon-MPC} in \eqref{thm:MPC-baseline-complete:e1:s1} and \eqref{thm:MPC-baseline-complete:e1:s2}; Moreover, we have applied the assumption that $\norm{A_{t_1}} \leq a$ and $\norm{B_{t_1}} \leq b$ and the triangle inequality in \eqref{thm:MPC-baseline-complete:e1:s3}. Since \eqref{thm:MPC-baseline-complete:e1} establishes a contraction for a deterministic state-action pair and the dynamics is deterministic, applying Lemma \ref{lemma:Wasserstain-preserves-contraction} finishes the proof of the first conclusion of Theorem \ref{thm:W-Robustness-and-ROB}.

Using a similar decomposition technique with \cite{lin2022bounded}, by Lemma \ref{lemma:finite-horizon-MPC}, we see that the trajectory $\left(x_t:t\in [\nt]\right)$ of \ouralg satisfies that
\begin{align}\label{thm:MPC-baseline-complete:e2}
    \norm{x_t} \leq \sum_{t' = 0}^{t-1}\norm{\Phi_{t', t}^{k}} \cdot (\norm{w_{t'}} + b \overline{R}) \leq C\sum_{t' = 0}^{t-1} \widehat{\lambda}^{t - t'}(d + b \overline{R}) \leq \frac{C(d + b\overline{R})}{1 - \widehat{\lambda}},
\end{align}
where we denote $\Phi_{t', t}^{k} \coloneqq (A_{t'-1} + B_{t'-1} K_{t'-1}^{k}) \cdots (A_t + B_t K_t^{k})$ and the assumption that \ouralg deviates at most $\overline{R}$ from $\MPC_k$'s action. We also see that
\begin{align}\label{thm:MPC-baseline-complete:e3}
    \norm{u_t} \leq \norm{K_t^{k} x_t} + \overline{R} \leq C \overline{R}_x + \overline{R}.
\end{align}
This finishes the proof of the second statement in Theorem \ref{thm:W-Robustness-and-ROB}.

Let the trajectory of $\MPC_k$ when executed without the machine-learned advice be denoted by $\left(\overline{x}_t : t \in [\nt]\right)$. We see that
\begin{align*}
    \norm{\overline{x}_t} = \norm{\sum_{t' = 0}^{t-1} \Phi_{t', t}^{k} w_{t'}} \leq C \sum_{t'=0}^{t-1} \widehat{\lambda}^{t - t'} \norm{w_{t'}}.
\end{align*}
Applying the Cauchy-Schwarz inequality, we obtain
\begin{align}\label{thm:MPC-baseline-complete:e4}
    \norm{\overline{x}_t}^2 &\leq \left(C \sum_{t'=0}^{t-1} \widehat{\lambda}^{t-t'} \norm{w_{t'}}\right)^2 \leq C^2\left(\sum_{t'=0}^{t-1} \widehat{\lambda}^{t-t'}\right)\left(\sum_{t'=0}^{t-1} \widehat{\lambda}^{t-t'} \norm{w_{t'}}^2\right)\nonumber\\
    &\leq \frac{C^2}{1 - \widehat{\lambda}} \sum_{t'=0}^{t-1} \widehat{\lambda}^{t-t'} \norm{w_{t'}}^2.
\end{align}
For the control actions of $\MPC_k$, we also see that
\begin{align}\label{thm:MPC-baseline-complete:e5}
    \norm{\overline{u}_t}^2 = \norm{K_t^{k} \overline{x}_t}^2 \leq C^2 \norm{\overline{x}_t}^2 \leq \frac{C^4}{1 - \widehat{\lambda}} \sum_{t'=0}^{t-1} \widehat{\lambda}^{t-t'} \norm{w_{t'}}^2.
\end{align}
Therefore, we get the following bound on the total cost:
\begin{subequations}\label{thm:MPC-baseline-complete:e6}
\begin{align}
    J(\MPC_k) ={}& \sum_{t=0}^{\nt-1} \left(\frac{1}{2}(\overline{x}_t)^\top Q_t \overline{x}_t + \frac{1}{2}(\overline{u}_t)^\top R_t \overline{u}_t\right)\nonumber\\
    \leq{}& \frac{\ell}{2} \sum_{t=0}^{\nt-1}\left(\norm{\overline{x}_t}^2 + \norm{\overline{u}_t}^2\right)\label{thm:MPC-baseline-complete:e6:s1}\\
    \leq{}& \frac{C^2(1 + C^2)}{2(1 - \widehat{\decayfactor})}\sum_{t=0}^{\nt-1} \sum_{t'=0}^{t-1} \widehat{\lambda}^{t - t'} \norm{w_{t'}}^2\label{thm:MPC-baseline-complete:e6:s2}\\
    \leq{}& \frac{C^2(1 + C^2)}{2(1 - \widehat{\decayfactor})^2}\sum_{t=0}^{\nt-2} \norm{w_t}^2,\nonumber
\end{align}
\end{subequations}
where we have used the assumption that $Q_t \preceq \ell I$ and $R_t \preceq \ell I$ in \eqref{thm:MPC-baseline-complete:e6:s1}; we have also used the inequalities \eqref{thm:MPC-baseline-complete:e4} and \eqref{thm:MPC-baseline-complete:e5} in \eqref{thm:MPC-baseline-complete:e6:s2}. Combining \eqref{thm:MPC-baseline-complete:e6} with the lower bound of $J^*$ in Lemma \ref{lemma:opt-lower-bound} finishes the proof of the third Statement in Theorem~\ref{thm:W-Robustness-and-ROB}.

\section{Proof of Theorem~\ref{thm:MDP-baseline-W-Robustness}}\label{appendix:MDP-proof}
Before showing Theorem~\ref{thm:MDP-baseline-W-Robustness}, we first state a technical lemma that establishes the relationship between the TV distance and the Wasserstein distance.

\begin{lemma}\label{lemma:TV-to-Wasserstein}
For any distributions $\mu, \nu$ on $\mathcal{X}$, we have
\[W_1(\mu, \nu) = 2\TV(\mu, \nu) = \norm{\mu - \nu}_1.\]
\end{lemma}
\begin{proof}
To see this, note that since $\norm{x - x'}_1 = 2$ for any $x \not = x'$, the Wasserstein $1$-distance $W_1(\mu, \nu)$ equals 2 times the probability mass we need to transport to convert $\mu$ to $\nu$. For every $i \in \{1, \ldots, n\}$ such that $\mu_i > \nu_i$, we need to move out exactly $(\mu_i - \nu_i)$ from the probability mass at $e_i$ to other points $\left(e_j:j \not = i\right)$. Therefore, we must have
\[W_1(\mu, \nu) = 2 \sum_{i = 1}^n \mathbf{1}(\mu_i > \nu_i) \cdot (\mu_i - \nu_i) = 2\TV(\mu, \nu) = \norm{\mu - \nu}_1.\]
\end{proof}

Note that the MDP's transition kernel acts as a deterministic function. It maps the current state-action pair from $\mathcal{X} \times \mathcal{U}$ to the distribution of the subsequent state in $\mathcal{X}$.
Hence, the current state-action distribution on $\mathcal{X} \times \mathcal{U}$ maps to a distribution on $\Delta(\mathcal{X})$.
To proceed with this recursion, we require the distribution of the next state, which should be on $\mathcal{X}$. This is in contrast to needing the distribution of the distribution of the next state, which would be on $\Delta(\mathcal{X})$.
Therefore, to convert the distributions on $\Delta(\mathcal{X})$ back to distributions on $\mathcal{X}$, we require the following lemma.

\begin{lemma}\label{lemma:push-forward-distribution}
Let $\mu,\mu'$ be two distributions on $\Delta(\mathcal{X})$. It follows that $\norm{\mathbb{E}[\mu] - \mathbb{E}[\mu']}_1 \leq W_1(\mu, \mu')$.
\end{lemma}
Note that $\mathbb{E}[\mu]$ and $\mathbb{E}[\mu']$ are distributions on $\mathcal{X}$.
\begin{proof}
By the definition of the Wasserstein distance, we have
\[W_1(\mu, \mu') = \inf_J \int \norm{x - y}_1 \mathrm{d} J(x, y),\]
where $J$ is a joint distribution on $\Delta(\mathcal{X})\times \Delta(\mathcal{X})$ with marginals $\mu$ and $\mu'$. For any such joint distribution $J$, we have
\[\int \norm{x - y}_1 \mathrm{d} J(x, y) \geq \norm{\int (x - y) \mathrm{d}  J(x, y)}_1 = \norm{\mathbb{E}[\mu] - \mathbb{E}[\mu']}_1.\]
This finishes the proof of Lemma~\ref{lemma:push-forward-distribution}.
\end{proof}

We now resume our discussion with the proof of Theorem~\ref{thm:MDP-baseline-W-Robustness}.

Given the state-action distribution $\rho$ at step $t$, let $\mu_{t:t'}(\rho)$ denote the resulting state distribution at step $t'$. We slightly abuse the notation so that for any pair $(x, u) \in \mathcal{X}\times \mathcal{U}$, $\mu_{t:t'}((x, u))$ still outputs the resulting state distribution at step $t'$. We see that
\[\norm{\mu_{t:t+1}((x, u)) - \mu_{t:t+1}((x', u'))}_1 \leq \norm{x - x'}_1 + \norm{u - u'}_1.\]
Therefore, by Lemmas~\ref{lemma:Wasserstain-preserves-contraction} and \ref{lemma:push-forward-distribution}, we see that
\begin{align}\label{thm:MDP-baseline-W-Robustness:e1}
    W(\mu_{t_1:t_1+1}(\rho), \mu_{t_1:t_1+1}(\rho')) \leq W(\rho, \rho').
\end{align}
Note that Lemma \ref{lemma:MDP-exp-contraction} and Lemma~\ref{lemma:TV-to-Wasserstein} imply that
\begin{align*}
    W(\mu_{t_1:t_2}(\rho), \mu_{t_1:t_2}(\rho')) \leq \decayfactor^{t_2 - t_1 - 1} W(\mu_{t_1:t_1+1}(\rho), \mu_{t_1:t_1+1}(\rho')).
\end{align*}
Combining this with \eqref{thm:MDP-baseline-W-Robustness:e1} gives that
\begin{align}\label{thm:MDP-baseline-W-Robustness:e2}
    W(\mu_{t_1:t_2}(\rho), \mu_{t_1:t_2}(\rho')) \leq \decayfactor^{t_2 - t_1 - 1} W(\rho, \rho').
\end{align}
For any distributions $\mu, \mu'$ on $\mathcal{X}$, we also see that $\norm{\bar{\pi}_{t_2}(\mu) - \bar{\pi}_{t_2}(\mu')}_1 \leq \norm{\mu - \mu'}$, which implies $$W(\bar{\pi}_{t_2}(\mu), \bar{\pi}_{t_2}(\mu')) \leq W(\mu, \mu').$$ Substituting this into \eqref{thm:MDP-baseline-W-Robustness:e2} gives that
\[W(\rho_{t_1:t_2}(\rho), \rho_{t_1:t_2}(\rho')) \leq 2 \decayfactor^{t_2 - t_1 - 1} W(\rho, \rho'),\]
validating that the Wasserstein robustness (Definition \ref{def:robust}) is satisfied.

\end{document}